%% file: main.tex
\documentclass{article}
\usepackage{iclr2027_conference,times}

\usepackage{graphicx}%
\usepackage{multirow}%
\usepackage{amsmath,amssymb,amsfonts}%
\usepackage{amsthm}%
\usepackage{mathrsfs}%
\usepackage{xcolor}%
\usepackage{textcomp}%
\usepackage{booktabs}%
\usepackage{algorithm}%
\usepackage{algorithmicx}%
\usepackage{algpseudocode}%
\usepackage{listings}%
\usepackage{hyperref}%
\usepackage{url}%
\usepackage{fontawesome5}%
\usepackage{float}%

\newcommand{\ADC}{\operatorname{ADC}}
\newcommand{\DAC}{\operatorname{DAC}}
\usepackage{wrapfig}
\usepackage{cleveref}

\usepackage{minitoc}
\usepackage{amsmath,amssymb,amsthm,mathtools}

\newtheorem{theorem}{Theorem}
\newtheorem{lemma}[theorem]{Lemma}
\newtheorem{corollary}[theorem]{Corollary}
\newtheorem{remark}[theorem]{Remark}
\newtheorem{assumption}{Assumption}
\crefname{figure}{Figure}{Figures}
\Crefname{figure}{Figure}{Figures}
\crefname{table}{Table}{Tables}
\Crefname{table}{Table}{Tables}
\crefname{section}{Section}{Sections}
\Crefname{section}{Section}{Sections}
\crefname{subsection}{Section}{Sections}
\Crefname{subsection}{Section}{Sections}
\crefname{subsubsection}{Section}{Sections}
\Crefname{subsubsection}{Section}{Sections}
\crefname{equation}{Equation}{Equations}
\Crefname{equation}{Equation}{Equations}
\crefname{algorithm}{Algorithm}{Algorithms}
\Crefname{algorithm}{Algorithm}{Algorithms}
\crefname{theorem}{Theorem}{Theorems}
\Crefname{theorem}{Theorem}{Theorems}
\crefname{lemma}{Lemma}{Lemmas}
\Crefname{lemma}{Lemma}{Lemmas}
\crefname{corollary}{Corollary}{Corollaries}
\Crefname{corollary}{Corollary}{Corollaries}
\crefname{remark}{Remark}{Remarks}
\Crefname{remark}{Remark}{Remarks}
\crefname{assumption}{Assumption}{Assumptions}
\Crefname{assumption}{Assumption}{Assumptions}
\crefname{appendix}{Appendix}{Appendices}
\Crefname{appendix}{Appendix}{Appendices}
\crefname{lstlisting}{Listing}{Listings}
\Crefname{lstlisting}{Listing}{Listings}

\RequirePackage{acronym}
\acrodef{AIHWKIT}[\textsc{AIHWKit}]{IBM Analog Hardware Acceleration Kit}
\input{mysymbol.sty}

\title{Making Analog Training Scale: Co-Designing Mapping, Optimizer, and Converters}
\author{\!\!\textbf{Zhaoxian Wu}$^{1}$\,
\textbf{Tayfun Gokmen}$^{2}$\,
\textbf{Omobayode Fagbohungbe}$^{2}$\,
\textbf{T. Patrick Xiao}$^{3}$\,
\textbf{Tianyi Chen}$^{1}$ \\
\normalfont $^{1}$\,Cornell Tech and Cornell University \quad
$^{2}$\,IBM T. J. Watson Research Center \\
$^{3}$\,Sandia National Laboratories \\
    \texttt{zw868@cornell.edu},\quad
    \texttt{\{tgokmen, Omobayode.Fagbohungbe\}@us.ibm.com}, \\
    \texttt{txiao@sandia.gov},\quad
    \texttt{tianyi.chen@cornell.edu} \\
    \faGithub\ \href{https://github.com/Zhaoxian-Wu/scalable-analog-training}{\texttt{github.com/Zhaoxian-Wu/scalable-analog-training}} \\
}
\iclrfinalcopy

\begin{document}
\maketitle
\lhead{Preprint}
\begingroup
\renewcommand{\thefootnote}{}
\footnotetext{The work was supported by the National Science Foundation Projects 2532349 and 2532653.
}
\endgroup

\doparttoc %
\faketableofcontents %

\begin{abstract}
Analog in-memory computing (AIMC) offers an alternative for model training by executing matrix operations directly where weights are stored. However, scaling AIMC to train modern deep models remains an open challenge due to severe hardware non-idealities, including physical weights with finite dynamic range and write granularity, analog-digital converters with finite resolution, and noisy and asymmetric updates.
Guided by the insight that gradient accumulation is sensitive to precision and rounding errors, we adopt a mixed-precision training paradigm: executing forward and backward matrix multiplications in the analog domain while computing weight gradients in the digital domain.
To enable scalable training, we present a holistic system-algorithm co-design that co-optimizes weight mapping to ensure well-conditioned physical and logical weight profiles, couples a preconditioned optimizer with threshold-triggered open-loop pulsing to stabilize training trajectories, and aligns converter dynamic ranges to suppress quantization errors.
Evaluated via hardware-calibrated architectural simulations calibrated with electrochemical RAM measurements, our framework scales Transformer training up to $123\text{M}$ parameters with validation loss scaling as $L\propto N^{-0.231}$, where $N$ is the parameter count, comparable to $L\propto N^{-0.238}$ for digital training.
\end{abstract}
\section{Introduction}
\label{sec:introduction}

\begin{wrapfigure}[17]{r}{0.46\textwidth}
\centering
\vspace{-1em}
\includegraphics[width=0.95\linewidth]{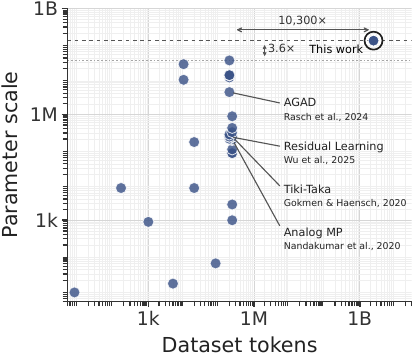}
\vspace{-1.1em}
\caption{
  Scaling landscape of analog training, where each point is one work.
  Labels identify representative work.
}
\label{fig:prior-training-scale}
\end{wrapfigure}
Training modern deep models requires moving large volumes of weights and activations between memory and compute units, making data movement a major contributor to energy and latency costs.
Analog in-memory computing (AIMC) offers a way to reduce this traffic by storing weights as conductances in non-volatile memory crossbars and using Ohm's and Kirchhoff's laws to execute matrix--vector multiplication (MVM) where the parameters reside~\citep{chen2013comprehensive,sze2017efficient,haensch2019next,sebastian2020memory,le202364}.
The same array can support the forward operation and its transpose in backpropagation \citep{agarwal2016resistive}, 
and showing more than tens or even hundreds of times better energy efficiency \citep{xiao2021parasitic,jain2019neural,cosemans2019towards,papistas202122}.

Despite its efficiency potential, AIMC hardware is subject to non-idealities, including finite converter precision (DAC/ADC), limited conductance range and update granularity, and noisy asymmetric programming, all of which can degrade model accuracy~\citep{gokmen2018lstm,nandakumar2020mixed,chen2023open,momeni2025training}.
At inference time, the weights are fixed, so aggressive accuracy-recovery measures can be tailored to the trained model and target workload~\citep{mackin2022optimised,rasch2023hardware}.
During training, however, weight and activation distributions evolve across iterations, while errors from repeated updates affect subsequent forward and backward passes.
Consequently, measures tuned to a fixed model cannot be applied unchanged throughout optimization, making stable analog training harder to scale.
Figure~\ref{fig:prior-training-scale} summarizes the model and task scales of representative AIMC training studies.
Most studies that execute iterative training with AIMC evaluate models with at most a few million parameters.
Representative studies at the tens-of-millions scale either focus on fine-tuning~\citep{fagbohungbe2025transfer} or assume precise control of weight states~\citep{ankit2020panther}.

\begin{figure*}[t]
\centering
\includegraphics[width=\textwidth]{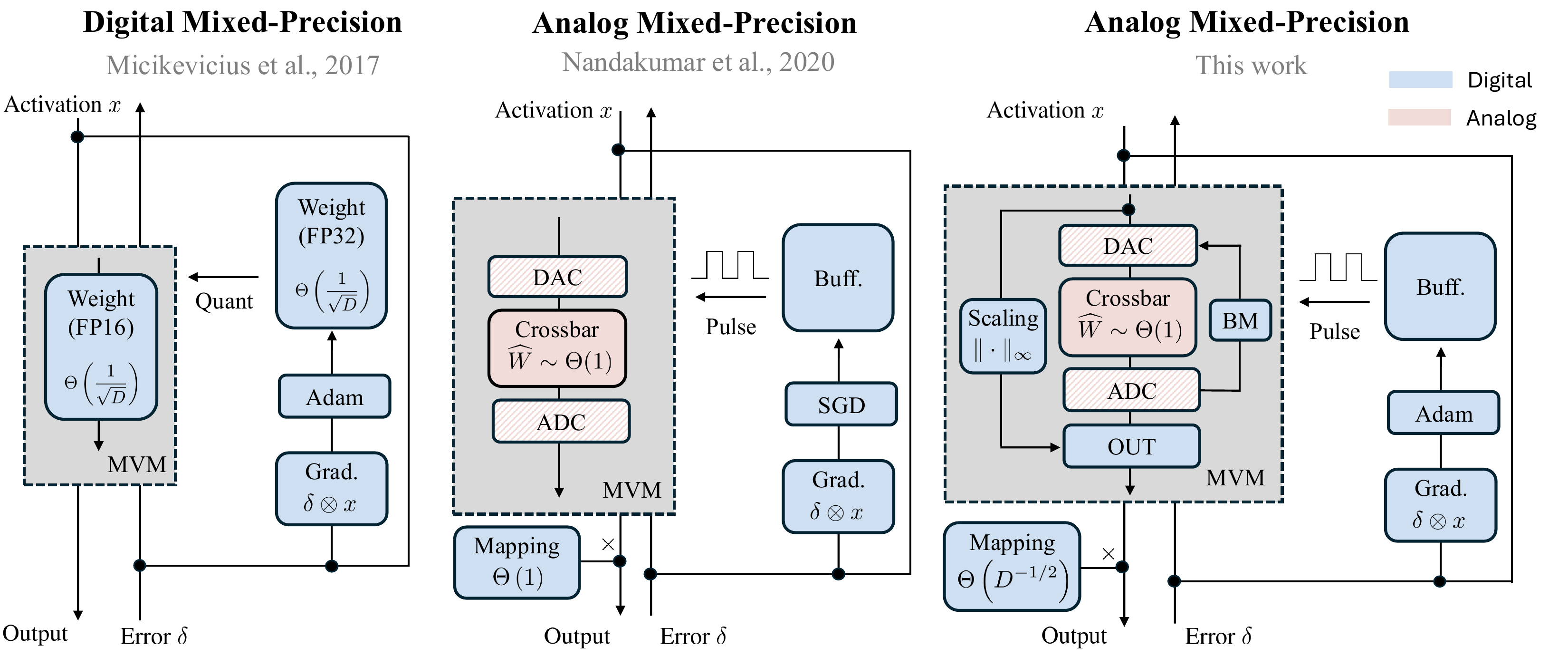}
\vspace{-0.2cm}
\caption{Comparison of training partitions.
\textbf{Left:} conventional digital mixed-precision training stores weights and executes the optimizer digitally, using reduced precision for selected arithmetic \citep{micikevicius2018mixed}.
\textbf{Middle:} prior analog mixed-precision training~\citep{nandakumar2020mixed} assigns MVMs to a DAC--crossbar--ADC path and uses a digital buffer to accumulate an SGD update before pulsed programming.
\textbf{Right:} On top of the conventional analog mixed-precision design, our proposed architecture couples logical-to-conductance mapping with DAC/ADC range alignment (bound management, BM) at the MVM boundary, allowing modern logical-space optimizers such as Adam to stabilize the training process.}
\label{fig:overview-comparison}
\vspace{-0.2cm}
\end{figure*}

Taken together, these results reveal a fundamental dilemma: while AIMC excels at MVMs, scaling it to modern deep architectures (e.g., Transformers) is fundamentally obstructed by two coupled challenges. 
First, adaptive optimizers (e.g., AdamW) are widely used, yet prior analog frameworks are tailored to SGD-like first-order updates, failing to handle complex preconditioned updates under noisy, asymmetric programming. 
Second, as model hidden dimensions scale up, the widening dynamic range of inner-product accumulations fundamentally clashes with the fixed physical ranges of analog crossbars and data converters (DACs/ADCs), causing forward signals and backpropagated gradients to clip severely or suffer quantization noise.
This motivates the core question:

\begin{center}
    \textbf{Q)} {\em Can an AIMC architecture scale on-chip training to larger models despite hardware-imposed limits on numerical precision?}
\end{center}

To address this challenge, we present a mixed-signal training framework that systematically reconciles modern logical-space optimization with physical crossbar constraints. 
Rather than forcing analog conductance to emulate complex optimizer states directly, our core insight is to cleanly partition the training loop: data-intensive forward and backward MVMs remain in the energy-efficient analog crossbars, while optimizer dynamics stay strictly in digital logical coordinates. 
To bridge these two domains, a digital residual buffer accumulates sub-threshold updates and triggers open-loop programming pulses only when updates exceed device granularity, extending the mixed-precision principle to arbitrary adaptive optimizers.
Crucially, we couple this logical-to-conductance mapping with dimension-aware converter range management (bound management) at the MVM boundaries. This co-design ensures that as model dimensionality scales, signal variances across the DAC--crossbar--ADC pipeline remain stably bounded within the hardware rails without sacrificing conductance resolution.

\textbf{Our contributions.} Our contributions can be summarized as follows:
\begin{itemize}
    \item[\textbf{C1)}] \textbf{Digital weight-gradient accumulation as a scaling principle.}
    We identify weight-gradient (W-grad) formation as a precision-critical bottleneck in AIMC training: gradient cancellation amplifies finite-precision accumulation errors, while fragmenting a fixed-batch update into frequent quantized full-rank transfers repeatedly incurs pulse-count truncation, update-management effects, and finite-state programming error.
    Motivated by these mechanisms, our architecture retains forward and backward MVMs in analog arrays but performs W-grad formation, optimization, and residual accumulation digitally, transferring only consolidated updates through threshold-triggered open-loop pulses.
    This partition provides an optimizer-agnostic logical update interface for modern adaptive training.

    \item[\textbf{C2)}] \textbf{Co-designed weight mapping, converter management, and theoretical bounds.}
    To preserve signal fidelity across varying model dimensions, we present a unified framework that couples logical-to-conductance mapping with dynamic converter range alignment (bound management). 
    This co-design reconciles variance-preserving initialization with the physical conductance limits of crossbars while preventing signal saturation and quantization collapse at the DAC/ADC boundaries. 
    Theoretically, we prove that our dynamic max-alignment policy is asymptotically optimal within the norm--rail converter family, providing quantifiable error bounds under finite conversion precision.

    \item[\textbf{C3)}] \textbf{Validation and preliminary empirical scaling trends.}
    We evaluate our framework through architecture-level simulations in IBM AIHWKit~\citep{rasch2021aihwkit}, using device models calibrated to experimental ECRAM measurements and decoder-only Transformers with up to $123.6\text{M}$ parameters following the Chinchilla scaling protocol~\citep{hoffmann2022training}, under full non-idealities (asymmetric responses, discrete write granularity, and DAC/ADC with IO noise, limited resolution and range).
    Despite a persistent validation-loss gap relative to digital AdamW, preliminary fits across five model scales in parameter count $N$ yield $L \propto N^{-0.231}$ for analog training and $L \propto N^{-0.238}$ for digital AdamW.
    Component ablations and a seven-optimizer compatibility study characterize the roles of optimizer choice, weight mapping, converter alignment, and the evaluated hardware non-idealities.
\end{itemize}

\subsection{AIMC Hardware Non-Ideality Model}
\label{sec:aimc-hardware-model}
To faithfully reflect physical crossbar execution, we model the primary analog non-idealities that degrade operations. 

\noindent\textbf{Precision-limited DAC and ADC conversion.}
AIMC performs efficient MVM in the analog domain, requiring digital inputs to be converted into electrical signals before each array read and the resulting analog outputs to be converted back into digital values for subsequent processing~\citep{hu2016dot,shafiee2016isaac}.
DACs encode input values as voltages or pulses, whereas ADCs digitize the accumulated array outputs.
Although these converters perform physically distinct operations, we can formulate their mappings as a uniform saturating quantizer.
Let $z\in\mathbb R$ be the scalar signal, $C>0$ be the \textit{quantization rail} defining the range $[-C,C]$, and $K\geq1$ be the integer number of equal quantization intervals. With $\Delta_{K,C}=2C/K$ as the quantization step, the uniform saturating quantizer is defined as
\begin{equation}
Q_{K,C}(z)
=-C+\Delta_{K,C}\,
\operatorname{clip}\!\left(\operatorname{round}\!\left(
\frac{z+C}{\Delta_{K,C}}\right),0,K\right),
\qquad z\in\mathbb R.
\label{eq:uniform-converter-quantizer}
\end{equation}
The operator $\operatorname{round}(\cdotc)$ rounds to the nearest integer, with ties resolved to the nearest even integer, and $\operatorname{clip}(t,0,K)=\min\{K,\max\{0,t\}\}$ enforces saturation at the rails.
For a vector $x\in\mathbb R^D$, we extend the quantizer coordinatewise by $[Q_{K,C}(x)]_d=Q_{K,C}(x_i)$ for $d=1,\ldots,D$.
In practice, physical DAC/ADC are configured with a fixed hardware input range $C_{\DAC}$ and $C_{\ADC}$. Therefore, the process can be formulated as $\DAC(z) = Q_{K, C_{\DAC}}(z)$ and $\ADC(z) = Q_{K, C_{\ADC}}(z)$ for input $z$.

\noindent\textbf{Matrix representation and noisy crossbar output.}
AIMC hardware uses crossbar arrays consisting of resistive elements to store matrix $\widehat W\in\reals^{D\times D}$, whose elements are represented by their conductances.
For a given voltage vector $x\in\reals^{D}$, the array produces a current vector $y\in\reals^{D}$ through the linear map $y=\widehat W x+\xi$, where $\xi$ is the circuit noise \citep{hu2016dot,jain2019neural,cosemans2019towards}.
Subtracting a reference-array response permits signed weights even though individual physical conductances are non-negative~\citep{nandakumar2020mixed,gong2018signal}.
Built on that, the logical weight matrix can be expressed as $W=s(\widehat W-\widehat W^\diamond)$, where $\widehat W^\diamond$ is the reference array, and $s\ge 0$ is the logical-to-conductance mapping scale.
Composing the processes in the crossbar array and DAC/ADC yields the vanilla implementation of the elementary operation, \textit{analog MVM}; see the middle figure in Figure~\ref{fig:overview-comparison}.

\noindent\textbf{Pulsed updates, asymmetry, and granularity.}
The conductance of resistive elements changes when receiving positive or negative electrical pulses.
For a conductance $\widehat W$, the nominal single-pulse change is $+\Delta w_{\min}q_+(\widehat W)$ or $-\Delta w_{\min}q_-(\widehat W)$ as the response, where $\Delta w_{\min}$ and $q_\pm(\cdotc)$ are device-specific \textit{write granularity} and state-dependent response functions, respectively \citep{wu2025analog}. Pulse polarity determines the sign, while conductance state determines the response magnitude.
Unequal positive and negative response magnitudes produce \textit{update asymmetry}; their equality defines a symmetric operating point, which is usually used as the reference array $\widehat W^\diamond$~\citep{kim2019zero}.
Additionally, the conductance typically saturates at device-specific boundaries so that it is always bounded in a \textit{dynamic range} $[w_{\min},w_{\max}]$.
Finite granularity limits the direct realization of small requested updates, while state dependence and asymmetry make the realized change differ from a nominal pulse command~\citep{burr2015experimental, agarwal2016resistive, chen2015mitigating}.

\subsection{Related Work}
\label{sec:related-work}
Pioneered by early crossbar circuits and heterogeneous architectures~\citep{hu2012bsb,liu2014heterogeneous,hu2016dot,shafiee2016isaac}, extensive efforts have sought to bridge the gap between the promise of AIMC and its practical deployment, but its practical use remains constrained by hardware non-idealities, including device variability, array parasitics, and finite converter precision~\citep{chen2015mitigating,xiao2021parasitic,gong2018signal}.
Beyond accelerating inference, researchers are therefore exploring AIMC hardware that accelerates the complete training loop by executing forward and backward matrix operations on the array~\citep{agarwal2016resistive, gokmen2016acceleration, nandakumar2020mixed}.
However, training is more sensitive to these non-idealities than inference since the perturbations accumulate over many optimization steps and ultimately degrade model performance~\citep{gokmen2018lstm}.
To address this issue, several works have proposed solutions to mitigate the impact of hardware non-idealities.
For example, Tiki-Taka and its variants \citep{gokmen2020tiki,rasch2024fast} propose a compensation scheme to reduce bias from asymmetric conductance changes, and Residual Learning extends this idea and introduces a residual mechanism \citep{wu2025analog,li2026multitile}.
Among them, analog mixed-precision training separates low-precision analog MVMs from high-precision digital W-grad to stabilize training~\citep{nandakumar2020mixed}.
This paper complements these works by focusing on the co-design of the optimizer, mapping, and converter management to further stabilize training, as well as the scalability of mixed-precision AIMC systems.

\section{Holistic Co-Design for Mixed-Signal AIMC Systems}
\label{sec:method}

We consider training in a mixed-signal AIMC system~\citep{nandakumar2020mixed}, where forward and transposed backward MVMs run in the analog domain, while W-grad computation, nonlinear operations, and optimization remain digital.
For the $b$-th sample and a linear layer with weight matrix $W\in\reals^{D\times D}$, the forward pass computes $y_b=Wx_b$ from the activation $x_b\in\reals^{D}$, while the backward pass computes $W^{\mathsf T}\delta_b$ from the output error $\delta_b\in\reals^{D}$.
Having obtained $B$ pairs of activation and error signals, we compute \textit{weight gradient (W-grad)} via $G=B^{-1}\sum_{b=1}^{B}\delta_b x_b^{\mathsf T}$, which involves $B$ outer products and accumulation of the resulting matrices.

\noindent\textbf{Analog MVM operation.}
The baseline analog mixed-signal architecture proposed by \cite{nandakumar2020mixed} employs fixed DAC and ADC quantization rails.
Inspired by \cite{gokmen2017cnn}, we dynamically adapt these rails to reduce quantization distortion to accommodate such distribution shifts.
Because physical converters operate over fixed input ranges, adjusting the effective rail is equivalent to applying inverse digital scaling to the input before conversion.
To realize an effective dynamic rail $C$, we prescale the digital input to $x/C$, whose DAC conversion is the voltage vector applied to the crossbar.
Accounting for converter quantization and crossbar read noise $\xi$, the logical MVM is $y=sC\operatorname{ADC}\!\left((\widehat W-\widehat W^\diamond)\operatorname{DAC}(x/C)+\xi\right)$.
The quantity inside the ADC is the physical crossbar output current, while the reference-array subtraction and the scale $s$ convert the digitized response to logical coordinates.
Since the input to DACs can be freely scaled in the digital domain, we can set its input range to $[-1, 1]$, yielding $\operatorname{DAC}(z) = Q_{K, 1}(z)$.
In contrast, the input to ADCs is current produced directly by the crossbar and cannot be preconditioned digitally.
Consequently, setting the physical ADC rail $C_{\mathrm{ADC}}$, where $\operatorname{ADC}(z) = Q_{K, C_{\mathrm{ADC}}}(z)$, requires a more principled design.
As illustrated on the right of \cref{fig:overview-comparison}, this implements the basic analog MVM operation.

\noindent\textbf{Digital W-grad and optimizer.}
This work adopts the mixed-precision training paradigm with digital W-grad.
Beyond safeguarding numerical fidelity during accumulation, this digital handling allows the subsequent optimization step to be performed entirely in the digital domain, which decouples the choice of algorithm from restrictive analog programming primitives, readily unlocking advanced preconditioning techniques critical for modern networks, such as AdamW \citep{loshchilov2019decoupled}, Shampoo \citep{gupta2018shampoo}, Lion \citep{chen2023lion}, or Muon \citep{jordan2024muon, liu2025moonlight}.
At step $k$, the chosen optimizer produces a signed logical increment $\Delta W_k$.

\noindent\textbf{Transfer with open-loop pulsing.}
Translating these logical updates into physical conductance changes, however, exposes a key mismatch in precision and granularity.
Under the mapping design, this corresponds to an intended conductance increment $\Delta\widehat{W}_k = \Delta W_k / s$.
Under small learning rates, however, elements of $\Delta\widehat{W}_k$ frequently fall below the write granularity $\Delta w_{\min}$.
To avoid losing sub-threshold increments, we introduce an additional digital residual buffer $H_k$ to accumulate them until they can be realized by one or more physical pulses.
Once the elements in the buffer exceed the threshold, the controller transfers them to the crossbar array. In the digital domain, we compute the number of pulses $N_k$ to be applied to the array and update the residual buffer $H_k$:
\begin{equation}
  H_{k+\frac{1}{2}} = H_k+\Delta W_k,\quad
  N_k = 
    \operatorname{trunc}\!\left(
      \frac{H_{k+\frac{1}{2}}}{s\Delta w_{\min}}
    \right),
    \quad
    H_{k+1}=H_{k+\frac{1}{2}}-sN_k\Delta w_{\min}.
  \label{eq:mp-open-loop-pulse-count}
\end{equation}
where $\operatorname{trunc}(\cdotc)$ denotes truncation toward zero.
After that, $N_k$ pulses are applied to the crossbar array. 
Ideally, the conductance update should be $N_k\Delta w_{\min}$, but hardware non-idealities can cause the actual update to deviate.
In principle, write-and-verify can improve programming fidelity, but repeated reads and corrective pulses introduce additional latency and energy overhead \citep{yao2020fully}.
We therefore use open-loop programming without verifying each update.
Although each programming operation is open-loop, the training trajectory retains feedback: subsequent forward and backward passes use the conductance state produced by the preceding writes.
The next gradient is therefore evaluated at the perturbed iterate, allowing moderate programming errors to enter as optimization perturbations rather than requiring exact correction of every write.
We evaluate empirically whether this training-level feedback is sufficient, without claiming that individual write errors are measured or explicitly corrected.

Having established the end-to-end execution pipeline, we now step back to examine the fundamental trade-offs that govern this analog--digital partitioning.
In principle, W-grad could be mapped directly onto the crossbar using parallel analog rank-update \citep{gokmen2016acceleration}, drastically reducing the update complexity from $O(D^2)$ to $O(D)$.
\cite{nandakumar2020mixed} does not adopt this approach because it raises endurance concerns of crossbar devices.
The next section approaches this trade-off from an algorithmic and scalability perspective: tasks like Transformer training demand high-fidelity gradient representations to resolve accurate descent directions.

\begin{figure}[t]
\centering\vspace{-0.2cm}
\includegraphics[width=\linewidth]{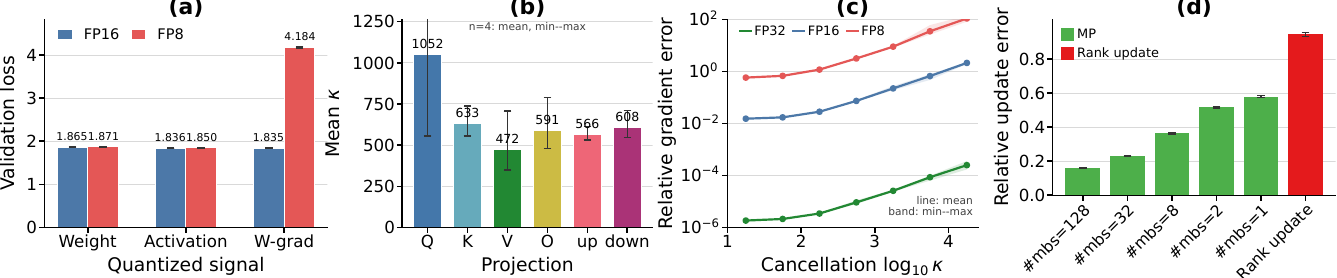}
\vspace{-0.4cm}\caption{
\textbf{(a)} Final validation loss of digital Shakespeare Transformer training with FP16 or FP8 weight, activation, or W-grad representations.
\textbf{(b)} Cancellation factor $\kappa=\sum_n|g_n|/|\sum_n g_n|$ of six projection types.
\textbf{(c)} Relative gradient error versus cancellation strength when the accumulator is rounded after every addition under FP32, FP16, and FP8.
\textbf{(d)} Single-step linear-layer update error at total batch size 128: green bars use mixed-precision (MP) digital-residual transfers at the indicated micro-batch sizes (\#mbs), while the red bar uses direct rank updates.
Bars and curves show four-seed means; whiskers and shaded bands span the observed seed min--max range.
}
\label{fig:precision-gradient-diagnostics}\vspace{-0.2cm}
\end{figure}

\subsection{Why Weight-Gradient Accumulation Should Remain Digital?}
\label{sec:digital-accumulation-scaling}
Two compounding factors lead us to adopt digital W-grad: severe numerical sensitivity during gradient reduction, and repeated write-path error when a fixed-batch update is fragmented into frequent quantized transfers.

\noindent\textbf{W-grad is sensitive to precision loss.}
Before addressing analog crossbars, we examine the fundamental precision limits of the training loop in a purely digital setting.
As shown in Figure~\ref{fig:precision-gradient-diagnostics}(a), the forward, backward, and weight-gradient paths exhibit drastically different numerical sensitivities.
While quantizing weight or activation representations to FP8 incurs almost no degradation in final validation loss compared to the FP16 baseline, applying identical FP8 quantization to the gradient path catastrophically impairs convergence.
Notably, in this diagnostic, intermediate error signals and parameter gradients are quantized while reduction operations remain in FP32.
This stark contrast demonstrates that gradient fidelity is uniquely fragile, underscoring that maintaining high precision during gradient handling is an unavoidable prerequisite for stable convergence.

\noindent\textbf{Gradient cancellation and accumulation amplification.}
The weight-gradient's acute sensitivity stems from \textit{numerical cancellation} during mini-batch accumulation.
For the $b$-th sample, let $g_b := \delta_b x_b^{\mathsf T}$ denote its outer-product contribution to a weight-gradient coordinate, and let $G := \sum_b g_b$ denote the net gradient.
To see how cancellation degrades precision, consider summing two opposite contributions $g_1 = +10.0$ and $g_2 = -9.9$: while each term has magnitude $\approx 10$, their true sum is a tiny residual $0.1$.
Under low-precision representations with limited mantissa bits, adding or rounding these large quantities introduces rounding noise that easily dwarfs the true residual $0.1$, or rounds it to zero entirely.
We quantify this effect via the cancellation factor $\kappa := \sum_b |g_b| / |\sum_b g_b|$, where $\kappa=1$ indicates identical signs across samples, and a large $\kappa$ signifies that the net update is merely a microscopic signed residual surviving vast opposing contributions.
The arithmetic mean of $\kappa$ is well above unity for each of the six projection types in the measured Transformer checkpoint (Figure~\ref{fig:precision-gradient-diagnostics}(b)), indicating substantial cancellation across these projection types.
As Figure~\ref{fig:precision-gradient-diagnostics}(c) shows, when sequential additions are subject to finite precision, the relative reconstruction error escalates sharply with $\kappa$: FP8 accumulation rapidly diverges from the true gradient, and even FP16 incurs considerable error at high cancellation regimes.
Hence, the core challenge is not merely storing an individual outer product, but reliably preserving this microscopic signed residual throughout mini-batch accumulation.

\noindent\textbf{Avoiding repeated quantized-transfer error.}
Digital W-grad also controls how frequently a consolidated update crosses the digital--analog boundary.
To isolate this effect, Figure~\ref{fig:precision-gradient-diagnostics}(d) fixes the total batch size ($B=128$), partitions the same ordered samples into micro-batches, and uses a symmetric finite-state soft-bounds device.
For MP updating, each exact micro-batch gradient enters a dense digital residual and then immediately requests a quantized full-rank tile transfer.
As the micro-batch becomes smaller, the same fixed-batch target is fragmented across more transfer events.
Each event reintroduces pulse-count truncation and capping, transfer-specific update-management effects, and finite-state programming error; these discrepancies accumulate over the fixed batch and increase the one-step realized-update error.
Larger micro-batches instead consolidate more contributions digitally and invoke the write path fewer times.
The direct rank-update result is invariant to the partition in this deterministic symmetric configuration and is therefore shown once.
Thus, panel (d) isolates transfer-fragmentation error in a controlled single-step setting; it neither relies on device asymmetry nor by itself establishes end-to-end training failure.

The remaining digital-to-analog transfer introduces physical constraints: finite write granularity, bounded conductance range, and programming non-idealities.
Next, we address how logical weights map to conductances, how optimizer updates transfer across the digital-analog boundary, and how finite-range data converters affect the analog MVMs.

\begin{figure}[t]
\centering
\includegraphics[width=0.8\linewidth]{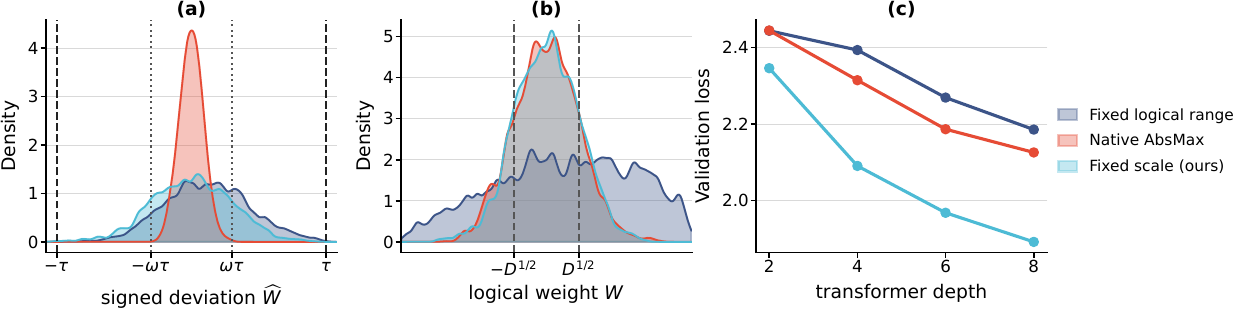}\vspace{-0.2cm}
\caption{
Initialization maps and analog Transformer training on the Shakespeare dataset.
\textbf{(a, b)} Signed conductance-deviation and logical-weight distributions for the fixed-logical-range control \citep{nandakumar2020mixed}, the native-AbsMax control \citep{rasch2020large}, and our proposed fixed-scale mapping.
\textbf{(c)} Final validation loss across model depth for the same mapping-based controls.
}\vspace{-0.2cm}
\label{fig:initial-native-versus-s-scaling}
\end{figure}

\subsection{Logical weight representation and physical mapping}
\label{sec:stable-weight-scaling}

A prerequisite for a stable and scalable analog training framework is a well-defined weight representation—specifically, how physical conductances are mapped to logical weights.
Because weights evolve in the logical domain driven by gradients and the optimizer, numerical stability across forward and backward passes remains a primary requirement. In conventional digital training, this is canonically achieved through variance-preserving initializations (e.g., Xavier or Kaiming init \citep{glorot2010initialization, he2015rectifier}) where weight standard deviations scale as $\Theta(D^{-1/2})$.
However, directly imposing this in the physical domain becomes problematic for wide layers: as target conductances shrink, they can fall below the write granularity $\Delta w_{\min}$ and waste the available dynamic range.
Conversely, device-friendly operation favors distributing initial conductances broadly across the available dynamic range around the symmetric reference point.
To reconcile these conflicting requirements, we introduce a scaling factor $s$ that decouples the layer-width-dependent logical initialization from the width-independent physical conductance occupancy.

For simplicity, we suppress the reference offset and let $\widehat W$ denote the signed deviation from the reference conductance, with coordinate range $[-\tau,\tau]$ centered at $0$; the underlying physical conductances remain non-negative.
On the physical side, we target a distribution centered at $0$ with pre-clipping standard deviation $\operatorname{Std}(\widehat W)=\tau/\omega$, where the inverse-occupancy factor $\omega>1$ leaves headroom to the device bounds.
This physical initialization scale is independent of layer width.
On the logical side, we impose the target initial logical-weight standard deviation $\sigma_w \propto D^{-1/2}$ by setting
$s={\omega\sigma_w}/{\tau}$.
Finite and asymmetric conductance bounds can truncate the physical initialization tails, so $\tau/\omega$ is a target pre-clipping standard deviation rather than an exact post-programming equality.
Figure~\ref{fig:initial-native-versus-s-scaling} (a) and (b) compare two mapping-based controls motivated by representative prior approaches: a fixed-logical-range control \citep{nandakumar2020mixed} and a native-AbsMax control \citep{rasch2020large}.
The fixed-logical-range control distributes signed conductance deviations broadly across the available dynamic range, but does not preserve the width-dependent logical initialization scale.
Conversely, the native-AbsMax control preserves the desired logical initialization scale, but confines most signed conductance deviations to a narrow range that is difficult to program accurately on hardware \citep{rasch2020large}.
As shown in \cref{fig:initial-native-versus-s-scaling} (c), both controls underperform the proposed fixed-scale mapping.

We now shift focus to the analog MVM, where moving signals through finite-range DACs and ADCs introduces additional distortion.

\subsection{Converter alignment by digital rescaling}
\label{sec:dac-range-alignment}
Since the DAC can be formulated as a uniform saturating quantizer, a natural way to reduce quantization error is to optimize the DAC rail to align with the input distribution.
We treat the distribution alignment as a joint choice of the vector statistic and the multiplier.
For $x \in \mathbb{R}^D$, we consider the vector-adaptive rail $A\|x\|_p$, where $\|\cdotc\|_p$ is the $\ell_p$-norm and $A>0$ is a multiplier.
The problem of finding the optimal rail then reduces to selecting the optimal norm order $p$ and multiplier $A$ that minimize the average reconstruction distortion $R_K(p,A;\mathcal V)$ for a given input distribution $\mathcal V$.
\begin{equation}
\begin{aligned}
r_K(\mathcal V)
&:=\inf_{p\in[1,\infty]}\inf_{A>0}
\left\{
  R_K(p,A;\mathcal V) := \frac{1}{D}\cdot
\mathbb E_{x\sim\mathcal V}\!\left[
\left\lVert Q_{K,A\lVert x\rVert_p}(x)-x\right\rVert_2^2
\right]
\right\}.
\end{aligned}
\label{eq:joint-norm-rail-distortion}
\end{equation}
Solving problem \eqref{eq:joint-norm-rail-distortion} exactly is generally challenging and expensive. 
Existing work has derived closed-form solutions for Gaussian or Laplace input distributions \citep{banner2019post} in the post-training phase.
However, during training, activation statistics may not remain stationary, so deploying a distribution-specific rule can require repeated distribution estimation and recalibration.
We therefore turn to the computationally efficient max-aligned $\ell_\infty$ rail, $(p,A)=(\infty,1)$, which requires computing the maximum absolute value of the input vector, \citep{rasch2020large}.
As shown below, this $\ell_\infty$ rail well approximates the distribution-specific optimum within the norm--rail family under a mild continuous assumption on the input distribution.

\begin{theorem}[Asymptotic optimality]
\label{thm:max-norm-alignment}
Suppose the joint distribution $\mathcal V$ is absolutely continuous with respect to Lebesgue measure on $\mathbb R^D$ and satisfies $\mathbb E_{x\sim \mathcal V}[\lVert x\rVert_2^2]<\infty$.
Then max alignment is asymptotically optimal within the norm--rail family:
$
\lim\limits_{K\to\infty}{R_K(\infty,1;\mathcal V)}/{r_K(\mathcal V)}=1.
$
\end{theorem}
The proof of Theorem~\ref{thm:max-norm-alignment} is deferred to Appendix~\ref{app:max-norm-alignment}.
For any fixed dimension and input distribution satisfying its assumptions, \Cref{thm:max-norm-alignment} shows that the ratio of the expected distortion of max alignment to the minimum distortion within the norm--rail family converges to one as converter precision increases.
The following theorem complements this fixed-dimension result by upper-bounding the dimension dependence of the max-aligned distortion under coordinatewise sub-Weibull tails, which include Gaussian and Laplace-type tails as special cases.

\begin{theorem}[Sub-Weibull max alignment]
\label{thm:sub-weibull-max-bound}
Suppose that, for some $\alpha,L_s>0$ and $M_s\geq1$, every coordinate satisfies $\Pr(|x_i|>x)\leq M_s\exp(-(x/L_s)^\alpha)$ for all $x>0$.
Then
\begin{equation}
  R_K(\infty,1;\mathcal V)\lesssim_{\alpha}\frac{(D-1)L_s^2(\log(2M_sD))^{2/\alpha}}{DK^2}
  = \ccalO\lp\frac{(\log(2M_sD))^{2/\alpha}}{K^2}\rp.
\end{equation}
Here, $a\lesssim_{\alpha} b$ means that $a\leq M_{\alpha}b$ for an $\alpha$-dependent constant $M_{\alpha}>0$.
\end{theorem}
The proof of Theorem~\ref{thm:sub-weibull-max-bound} is deferred to Appendix~\ref{app:max-norm-alignment}.
\Cref{thm:sub-weibull-max-bound} implies that the distortion of a max-aligned rail increases at most logarithmically with dimension $D$ for sub-Weibull distributions, including Gaussian ($\alpha=2$) and Laplace $O(\log^2 D)$ distributions as special cases.
Thus the maximum-coordinate contribution is $O(\log D)$ for sub-Gaussian tails ($\alpha=2$) and $O(\log^2 D)$ for Laplace-type tails ($\alpha=1$).
Therefore, the max-aligned rail offers a low-overhead, distribution-agnostic, and dimensionally robust approximation for the optimal DAC input alignment.

\begin{figure}[t]
\centering
\includegraphics[width=0.98\textwidth]{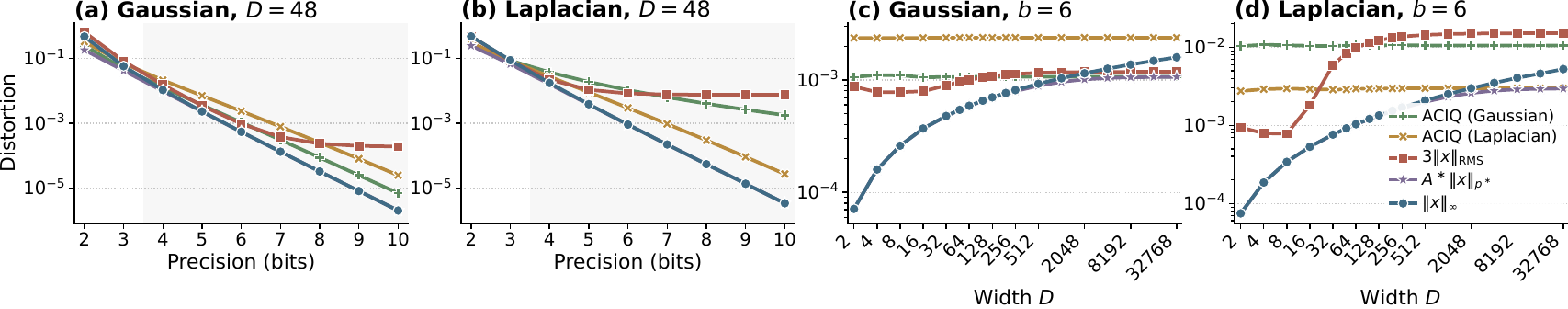}
\vspace{-0.2cm}
\caption{Reconstruction distortion under Gaussian and Laplacian distributions.
\textbf{(a)} and \textbf{(b)} show the effect of quantization precision.
The gray bands mark precisions at which the max-aligned rail is within 10\% of the optimal distortion.
\textbf{(c)} and \textbf{(d)} show the effect of width.
}
\label{fig:normalization-mmse}\vspace{-0.2cm}
\end{figure}

\paragraph{Empirical verification.}
Figure~\ref{fig:normalization-mmse} compares the reconstruction distortion of the candidate DAC rails.
In panels (a) and (b), the max-aligned curve converges toward, and at higher precision overlaps with, the calibrated optimum.
At $D=48$, from four bits onward, max alignment is within 10\% of the calibrated distortion for both product laws, whereas both fixed ACIQ \citep{banner2019post} comparators remain visibly separated from the optimum; at fixed precision, Appendix~\ref{app:normalization-mmse-calibration} shows a larger gap at sufficiently large widths.
This controlled finite-sample result is consistent with the asymptotic prediction of Theorem~\ref{thm:max-norm-alignment}.
Panels (c) and (d) show that the max-aligned distortion grows mildly with width, consistent with the tail-dependent dimension dependence in Theorem~\ref{thm:sub-weibull-max-bound}.
The ACIQ traces further show that a threshold optimized for one distribution can be suboptimal for another, which is particularly undesirable during training when activation distributions change.

\begin{figure*}[t]
\centering\vspace{-0.2cm}
\includegraphics[width=0.98\textwidth]{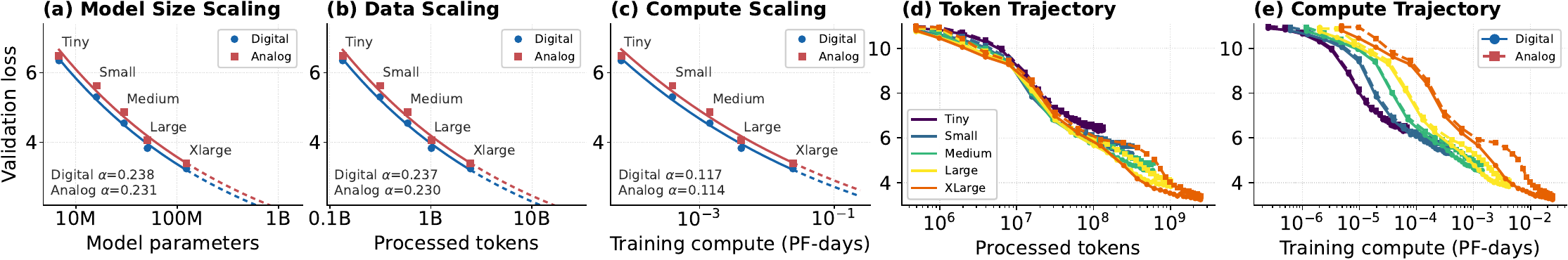}
\vspace{-0.2cm}
\caption{Scaling and training trajectories for the five completed OpenWebText schedules following the Chinchilla rule, comparing matched digital and analog training on Li-ECRAM tiles.
(a--c) Observed endpoints and preliminary power-law extrapolations fit $L=A X^{-\alpha}$ against parameter count $N$, processed tokens, and the full-model algorithmic compute proxy, respectively.
(d--e) Raw trajectories versus processed tokens and the shared compute proxy show the sample-efficiency and compute-normalized comparisons, respectively.
}
\label{fig:scaling}\vspace{-0.2cm}
\end{figure*}
\input{tables/nandakumar-mpsgd-values.tex}
\begin{table*}[t]
\centering\vspace{-0.2cm}
\setlength{\tabcolsep}{4pt}
\caption{Ablation study from MpSGD to our recipe.
}
\label{tab:nandakumar-sout-bridge}
\begin{tabular}{lrrrrrr}
\toprule
Scale & Params. & MpSGD & MpAdam & $+\mathrm{CA}$ & $+\mathrm{init}$ (ours) & Digital Adam \\
\midrule
Tiny   & 6.8M  & 10.7718   & 10.6035 & 6.8110 & \textbf{6.6075} & 6.3451 \\
Small  & 16.0M & 10.8545  & 10.5566 & 5.9143 & \textbf{5.6245} & 5.3033 \\
Medium & 29.9M & 10.8293 & 10.7397 & 5.2073 & \textbf{4.8696} & 4.5499 \\
\bottomrule
\end{tabular}\vspace{-0.2cm}
\end{table*}

\section{Scaling Up Analog In-Memory Training}
\label{sec:results-draft}

We evaluated whether the mixed-precision update path remains trainable as we scale a decoder-only transformer from 6.8M to 123.6M parameters.
Each model was trained from scratch on the OpenWebText dataset under matched schedules following the Chinchilla scaling rule~\citep{hoffmann2022training}.
We compared matched digital AdamW training with analog training on lithium-ion and metal-oxide electrochemical RAM (Li-ECRAM) tiles, using the proposed recipe in this work.

\noindent\textbf{Convergence and Scaling Laws.}
Figure~\ref{fig:scaling} illustrates the empirical scaling behavior and optimization trajectories.
Specifically, panels (a)-(c) display the scaling laws of analog training with respect to model size, processed tokens, and algorithmic training compute, while panels (d) and (e) show the loss convergence trajectories across tokens and compute.
At each evaluated scale, a persistent validation-loss gap exists between analog training and the digital Adam baseline.
However, under scaling, analog training achieves a power-law convergence rate comparable to digital training: the parameter-scaling fit yields $L \propto N^{-0.231}$, closely tracking $L \propto N^{-0.238}$ for digital training.
This demonstrates that despite device non-idealities and analog boundary noise, the mixed-precision architecture sustains predictable scaling dynamics across orders of magnitude.

\noindent\textbf{Bridging from baseline to This Work.}
To isolate contributors enabling robust scaling, we progressively bridge the baseline framework to our method.
We denote the SGD baseline in \cite{nandakumar2020mixed} as \textit{MpSGD}.
As reported in Table~\ref{tab:nandakumar-sout-bridge}, MpSGD fails during training across all scales.
MpAdam slightly mitigates the training collapse, but validation performance remains poor because converters lack alignment.
Incorporating converter alignment rectifies this degradation by jointly managing noise and dynamic range bounds, yielding a major performance gain.
Finally, pairing converter alignment with tailored weight initialization constitutes the complete configuration of this work, achieving stable convergence across Tiny through Medium models.

\section{Conclusion, Limitations, and Future Work}\label{sec:conclusion}
This paper identifies digital W-grad as a critical enabler for scaling mixed-signal AIMC training toward modern model regimes.
Our framework couples digital optimizer tracking with variance-preserving weight mapping and converter-range alignment.
In simulations calibrated to measured ECRAM characteristics, it supports training decoder-only Transformers from scratch on OpenWebText up to 123.6M parameters, establishing a first-order milestone: compounded modeled non-idealities do not inherently preclude scaling AIMC training dynamics to larger models.

\textbf{Limitations.}
The main scaling experiments use a single random seed, leaving run-to-run variability and uncertainty in the fitted scaling trends unquantified.
Although the simulated 16-bit DACs and 9-bit ADCs already introduce finite-resolution non-idealities, converters at these resolutions can incur substantial energy, area, and latency overheads in a physical implementation.
Because we do not model these costs, our results establish algorithmic and functional feasibility---a prerequisite for physical optimization---rather than a verified hardware-efficiency advantage.

\textbf{Future Work.}
Future work will align the simulations more closely with measured physical devices and pursue more effective, robust methods for handling device and circuit non-idealities, aiming to preserve scaling behavior at substantially lower DAC/ADC precision.
\section*{AI Use Statement}
In this work, we used generative AI tools to help develop theoretical models and conceptual frameworks, formulate mathematical claims, provide critical ingredients for proving mathematical claims, assist in writing proofs, propose or refine hypotheses, design or provide feedback on research methodology and experiments, implement methods, support qualitative and thematic data analysis, and interpret results.
We did not use generative AI tools to generate synthetic datasets, perform translation, or clean or reformat datasets; these tasks were not part of the AI-assisted workflow.
Additionally, we used generative AI tools to create or modify scientific figures or images, suggest experimental parameters, create or edit software code and other research artifacts, draft parts of this paper, summarize or analyze existing literature, identify research gaps, brainstorm, source and search for information, improve readability, identify relevant literature, format references, suggest paper structure, and propose titles or keywords.
Formulating survey or interview questions and transcribing recordings were not applicable to this work.
The authors reviewed all AI-assisted work by checking mathematical arguments and proofs, inspecting and testing code, validating analyses and reported results against experimental records, checking references against their sources, and reviewing figures and prose.
The authors take responsibility for the final content of this work, including all text, claims, and artifacts produced with the aid of generative AI.

\section*{Ethics Statement}
This work is an architecture-level simulation study and does not involve human participants, personal or sensitive data collection, or the release of a new dataset.
The study reports the modeled device and system assumptions, distinguishes simulated results from fabricated-hardware evidence, and states the limitations that could affect physical deployment.
The method could inform more energy-efficient training hardware, but it does not itself deploy such hardware or make claims about its safety, reliability, or societal effects outside the simulated setting.

\section*{Reproducibility Statement}
The assumptions and proofs for the converter-range analysis are stated in the main text and completed in the appendix.
The experimental protocol specifies the simulator, device and I/O settings, model configurations, datasets, evaluation metrics, and random seeds in the supplementary material.
Source code, configuration files, environment specifications, run logs, and the checkpoints underlying the reported results will be released publicly.
\bibliographystyle{iclr2027_conference}
\bibliography{
  bibs/refs,
  bibs/analog,
  bibs/publication
}
\newpage
\appendix

\begin{center}
  {\LARGE\bfseries Supplementary Material}
\end{center}
\vspace{-4em}
{
\renewcommand{\thepart}{}
\renewcommand{\partname}{}
\part{}
}
\parttoc 

\newpage

\section{Scaling landscape of analog training}
\label{app:training-scaling-landscape}

Table~\ref{tab:prior-training-scale} records the largest reported workload for which the analog-mapped weight count, training-data usage, architecture, training regime, and evidence can be determined among on-chip AIMC training studies.
Evidence H denotes a machine-learning task executed with physical hardware, whereas S denotes simulation only; physical hardware may still use a host computer, FPGA, or laboratory instrumentation for scheduling, gradient processing, routing, or pulse control.
The training regime is T for transfer learning from a pretrained model and F for training from scratch.
Throughout the table, the weight count measures distinct logical weights mapped to analog arrays for the stated workload, including fixed analog weights but excluding digitally held parameters, unused cells, differential-device multiplicity, auxiliary/reference arrays, and replicated mappings.
The table includes only experiments in which all learned matrix-multiplication and convolution weights of the evaluated network are mapped to analog arrays; experiments that map only a selected subset of network layers to analog while retaining other such layers digitally are excluded.
Counts omitted by the source are derived from reported mapped-layer dimensions or the corresponding public implementation, and token counts use only the training subset actually used rather than the complete upstream dataset; repeated epochs over a fixed dataset are not counted repeatedly, whereas schedule-based language-model training uses the number of token IDs actually processed.
For a like-for-like comparison, we exclude larger-scale experiments that leave one or more learned linear layers digital; an asterisk in the Work column identifies papers that also report such an excluded experiment.

\paragraph{Dataset size and number of tokens.}
Language tokens and image pixels do not carry commensurate amounts of information, so directly comparing their raw counts would be misleading.
For a relatively fair cross-modal comparison, we therefore express both language and image inputs in a common token-based unit: language examples are counted after tokenization, while images are converted to non-overlapping $16\times16$ patch tokens following the ViT convention~\citep{dosovitskiy2021image}, with input dimensions padded upward to the next patch boundary and no classification token, giving $\lceil H/16\rceil\lceil W/16\rceil$ tokens per image.
This equivalent token count is a scale proxy rather than a claim that one language token and one image patch contain exactly the same information.
Dataset size is the approximate compact input payload for one pass over the training subset actually used: image pixel-channels and binary features use 8-bit storage, numerical features use 32-bit storage, and language-model token IDs use 16-bit storage.
This storage proxy excludes preprocessing, augmentation, labels, and digital pretraining, and is not a training-compute metric.
For fixed datasets, repeated training epochs are not counted repeatedly.
For non-image datasets, \#tokens counts input scalars; for schedule-based OpenWebText training, it is the number of GPT-2 BPE token IDs processed by the completed schedule.
Dataset size and \#tokens use only the actual training subset or processed language-model stream, rather than the complete upstream dataset when the experiment uses a subset or a shorter schedule.

\paragraph{Prior scale frontier and algorithmic scope.}
The prior model frontiers are PANTHER's 34.0M-weight VGG-16 trained from scratch and the 26.6M-weight Swin-ViT transfer-learning simulation of \citet{fagbohungbe2025transfer}; full CIFAR-10 and CIFAR-100 are the largest prior fixed datasets by compact payload~\citep{ankit2020panther}.
PANTHER raises outer-product-accumulation precision by distributing weight bits heterogeneously across ReRAM slices and supports SGD and mini-batch SGD, but its evidence is simulated: bit slicing improves numerical resolution only, so accurate slice carries still rely on stable multilevel writes not demonstrated under realistic noise and nonlinearity.
Fagbohungbe et al. use c-TTv2, descended from two-array Tiki-Taka~\citep{gokmen2020tiki} and the offset-corrected variants of \citet{rasch2024fast}: an auxiliary analog array accumulates gradients, periodic transfers update the main weights, and sign chopping suppresses reference offsets.
Their Swin-ViT result is also simulation-only, starts from digital CIFAR-10 pretraining, fine-tunes on only two or five CIFAR-100 classes, and models transfer error as additive Gaussian noise.
The present work therefore moves the qualifying frontier from tens-of-millions-scale vision models and fixed image datasets to a 123.5M-weight decoder and a billion-token language-model training stream.

\paragraph{Why the three VGG-8 entries are distinct.}
The approximately 13.0M-weight VGG-8/CIFAR-10 rows provide different evidence: DNN+NeuroSim V2.0 is a software framework coupling PyTorch to abstract device, circuit, and architecture models~\citep{peng2021dnn}; \citet{kim2022cmos} calibrate simulation with a fabricated ferroelectric-transistor array that demonstrates selective programming and convolution; and \citet{chen2023open} additionally train a physical 60-weight bilayer classifier using linear, symmetric open-loop ECRAM updates, but retain activation, routing, and gradient calculation in software.
The complete VGG-8 loop remains simulated in all three studies.
Thus, the shared VGG-8 row is a common simulated model--dataset benchmark, not evidence of three equally scaled hardware training systems.

\paragraph{Hybrid hidden-weight training.}
The FeCAP--memristor design of \citet{martemucci2025ferroelectric} stores a frequently updated 10-bit signed hidden weight in binary FeCAPs and periodically transfers its sign and three most significant bits into a differential memristor pair, closely paralleling high-precision accumulation followed by sparse analog-weight transfer.
Its network results are simulations calibrated by measured transfer distributions, and the larger MobileNet-V2 studies freeze the feature extractor and train only a classifier of at most 256k analog weights; they are therefore excluded by the table's full-network mapping criterion.
The table instead records the qualifying end-to-end Fashion-MNIST MLP, whose four fully connected matrices comprise 187,800 analog-mapped weights.

\begin{table}[t]
\centering
\footnotesize

\setlength{\tabcolsep}{0.9pt}
\begin{tabular}{lcccrrr}
\toprule
Work & Evidence & Task & Model & \#analog wt. & Dataset/task & \#tokens \\
\midrule
\cite{li2021ecram} & H & F & Logic-gate net & $9$ & Truth tables & $8$ \\
\cite{ando2025transfer} & H & T & Binary classifier & $16$ & MNIST & $5$k \\
\cite{chen2023open} & H & F & Bilayer MLP & $60$ & Mushrooms & $81$k \\
\cite{ning2025index} & H & F & Sparse NN & $900$ & EMNIST L & $1.0$k \\
\cite{wang2019insitu} & H & F & CNN/RNN & $1.0$k & MNIST & $240$k \\
\cite{yao2020fully} & H & F & CNN & $2.8$k & MNIST & $240$k \\
\cite{li2018self} & H & F & MLP & $8.2$k & MNIST & $20$k \\
\cite{yi2023activity} & H & F & EN & $8.2$k & Braille words & $168$ \\
\cite{thunder2022variation} & S & F & MLP & $79$k & MNIST & $240$k \\
\cite{li2026multitile}$^{*}$ & S & F & LeNet-5 & $80$k & MNIST & $240$k \\
\cite{yilmaz2024neuro} & S & F & MLP & $102$k & MNIST & $240$k \\
\cite{burr2015experimental} & H & F & MLP & $165$k & MNIST & $20$k \\
\cite{martemucci2025ferroelectric}$^{*}$ & S & F & MLP & $188$k & Fashion-MNIST & $240$k \\
\cite{nandakumar2020mixed} & H & F & MLP & $199$k & MNIST & $240$k \\
\cite{ambrogio2018equivalent} & H & F & MLP & $205$k & CIFAR-100 & $200$k \\
\cite{gokmen2020tiki} & S & F & MLP & $235$k & MNIST & $240$k \\
\cite{gokmen2021} & S & F & MLP & $235$k & MNIST & $240$k \\
\cite{xiao2026dynamic}$^{*}$ & S & F & MLP & $235$k & MNIST & $240$k \\
\cite{wu2024towards}$^{*}$ & S & F & MLP & $235$k & CIFAR-10 & $200$k \\
\cite{falcone2025all} & S & F & MLP & $235$k & MNIST & $240$k \\
\cite{wu2025analog}$^{*}$ & S & F & MLP & $235$k & CIFAR-10 & $200$k \\
\cite{rasch2020large} & S & F & CNN & $270$k & CIFAR-100 & $200$k \\
\cite{xu2024trace} & S & F & SNN & $314$k & MNIST & $240$k \\
\cite{lim2019adaptive} & S & F & MLP & $415$k & MNIST & $240$k \\
\cite{cheng2019time} & S & F & MLP & $886$k & MNIST & $240$k \\
\cite{rasch2024fast} & S & F & ViT & $4.3$M & CIFAR-10 & $200$k \\
\cite{ando2025transfer} & S & T & ViT & $9.6$M & CIFAR-100 & $10$k \\
\cite{dang2026nova} & S & F & ResNet-18 & $11.2$M & CIFAR-100 & $200$k \\
\cite{peng2021dnn} & S & F & VGG-8 & $13.0$M & CIFAR-10 & $200$k \\
\cite{kim2022cmos} & S & F & VGG-8 & $13.0$M & CIFAR-10 & $200$k \\
\cite{chen2023open} & S & F & VGG-8 & $13.0$M & CIFAR-10 & $200$k \\
\cite{fagbohungbe2025transfer} & S & T & Swin-ViT & $26.6$M & CIFAR-100 & $10$k \\
\cite{ankit2020panther} & S & F & VGG-16 & $34.0$M & CIFAR-100 & $200$k \\
\midrule
\textbf{This work} & S & F & Transformer & $123.5$M & OpenWebText & $2.4$B \\
\bottomrule
\end{tabular}
\caption{Representative on-chip AIMC training studies. Evidence is H when the machine-learning task is executed with physical hardware and S for simulation only. Task is T for transfer learning from a pretrained model and F for training from scratch. Device characterization used to calibrate a simulation does not by itself qualify as H evidence. Studies reporting both forms of task-level evidence occupy separate H and S rows. The table includes only experiments in which all learned matrix-multiplication and convolution weights of the evaluated network are analog-mapped; experiments that retain any such network layers digitally and analog-map only selected layers are excluded. \#analog wt. counts distinct logical weights mapped to analog arrays for the stated workload, whether trainable or fixed; digitally held non-matrix parameters, unused capacity, differential-device multiplicity, auxiliary/reference arrays, and replicated mappings are excluded. ``Dataset/task'' gives the largest qualifying machine-learning workload with a determinable analog-mapped weight count. A superscript $*$ after a work indicates that its paper also reports a larger-scale experiment that leaves one or more learned linear layers digital; it is excluded to preserve a fair full-network analog-mapping comparison. EN denotes an equilibrium net.}
\label{tab:prior-training-scale}
\end{table}

The table reveals two separate gaps.
First, on-chip training architectures for modern networks are still evaluated predominantly in simulation and therefore remain exposed to mismatch between a finite device model and a particular deployed chip.
Second, the physical evidence for such architectures commonly uses external control and has not yet established a scalable integrated backward and optimizer path.
The goal of the present mixed-signal architecture is to close that second gap without requiring fragile analog circuitry to retain high-precision gradient history.

\section{Notation}
\label{app:notation}

This section fixes the mathematical conventions used in the converter-alignment analysis and distinguishes them from similarly named implementation parameters.

\paragraph{Random vectors, norms, and optimized distortion.}
Let $x=(x_1,\ldots,x_D)\in\mathbb R^D$ denote a $D$-dimensional random vector.
For $1\leq p<\infty$,
\begin{equation}
\lVert x\rVert_p=\left(\sum_{i=1}^D|x_i|^p\right)^{1/p},
\qquad
\lVert x\rVert_\infty=\max_{i\leq D}|x_i|.
\label{eq:notation-vector-norms}
\end{equation}
$\mathbb E [\cdotc]$ and $\Pr(\cdotc)$ denote expectation and probability, respectively.

\paragraph{Asymptotic notation.}
Unless stated otherwise, $O(\cdotc)$ and $o(\cdotc)$ refer to $K\to\infty$ at fixed $D$.
The relation $f(K)=O(g(K))$ means $|f(K)/g(K)|$ is eventually bounded, whereas $f(K)=o(g(K))$ means $f(K)/g(K)\to0$.
We write $f\lesssim_\alpha g$ when $f\leq C_\alpha g$ for a finite constant depending only on $\alpha$, and $f\asymp g$ when both $f\lesssim g$ and $g\lesssim f$ hold.

\section{Simulation Details and Additional Simulations}\label{app:simulation-details}

The language-model experiments use the decoder-only GPT; see \url{github.com/karpathy/nanogpt}.
For a batch of $B$ sequences with $T$ tokens, token indices are embedded by
\texttt{wte} and added to learned position embeddings from \texttt{wpe}; the
result has width $d$=\texttt{n\_embd}. The \texttt{wte} weight is tied to the
output matrix of \texttt{lm\_head}, so the vocabulary embedding and output
projection share one parameter tensor.

Each of the $L$=\texttt{n\_layer} pre-normalized decoder blocks contains two
residual sublayers. The attention sublayer applies \texttt{ln\_1}, followed by
the fused QKV projection \texttt{attn.c\_attn}: $d\rightarrow3d$. Its output
is split into queries, keys and values, reshaped into \texttt{n\_head} heads,
and processed with causal scaled dot-product attention; attention scores,
softmax and head concatenation remain digital. The attention result is then
projected by \texttt{attn.c\_proj}: $d\rightarrow d$, passed through dropout,
and added to the sublayer input. The second sublayer applies \texttt{ln\_2}
and the MLP \texttt{mlp.c\_fc}: $d\rightarrow4d$, \texttt{mlp.gelu}, and
\texttt{mlp.c\_proj}: $4d\rightarrow d$, followed by dropout and its residual
addition. A final \texttt{ln\_f} precedes \texttt{lm\_head}: $d\rightarrow V$,
which produces vocabulary logits used by the next-token cross-entropy loss.

\textbf{Transformer realization.}
For transformer workloads, the AIMC mapping is applied to the dense linear projections: the query, key, value, and attention-output projections, and the two feed-forward linear layers.
Each such projection executes its forward MVM on analog crossbar tiles, with digital-to-analog and analog-to-digital conversion at the interface described above.
The self-attention score computation, softmax, attention-weighted reduction, layer normalization, and other non-linear operations remain digital.

\textbf{Transformer-specific logical initialization}
For the nanoGPT transformer experiments, the general MVM-layer logical initialization standard deviation $\sigma_w \propto {1}/{\sqrt{D_{\mathrm{in}}}}$ is instantiated as
\begin{equation}
\sigma_w=0.02\sqrt{\frac{768}{D_{\mathrm{in}}}},\qquad
\sigma_{w,\mathrm{res}}(D_{\mathrm{in}},L)=\frac{0.02}{\sqrt{2L}}\sqrt{\frac{768}{D_{\mathrm{in}}}},
\label{eq:logical-weight-initialization}
\end{equation}
where $\sigma_{w,\mathrm{res}}$ is used for attention and MLP residual projections.
The prefactor preserves nanoGPT's width-768 convention, and the factor $1/\sqrt{2L}$ applies its residual-depth correction.
Position embeddings retain logical standard deviation $0.02$, because a one-hot lookup does not accumulate $D$ unit-RMS inputs.

This appendix collects the benchmark, evaluation, and simulator details for the experiments that test the theoretical results on synthetic and real datasets.
We use the open-source \ac{AIHWKIT} toolkit to simulate AIMC-hardware behavior~\citep{rasch2021aihwkit}; see \url{github.com/IBM/aihwkit}.

\subsection{Precision and gradient-accumulation diagnostics (Figure~\ref{fig:precision-gradient-diagnostics})}
\label{app:precision-gradient-diagnostic-settings}
Figure~\ref{fig:precision-gradient-diagnostics} combines a digital training control, two fixed-checkpoint gradient diagnostics, and one single-layer analog write diagnostic.
Panels (b) and (c) use a separately trained checkpoint with the same Transformer architecture as panel (a), whereas panel (d) uses synthetic inputs and output gradients.

\paragraph{(a) Digital precision sensitivity.}
The character-level Shakespeare model has eight decoder blocks, one attention head, width 48, context length 256, dropout 0.2, no bias terms, and an affine final layer normalization.
Each condition is trained for 5,000 AdamW updates with batch size 64 for each of four training seeds (1--4).
The learning rate has a 100-step linear warm-up, followed by cosine decay from $10^{-3}$ to $10^{-4}$, with $(\beta_1,\beta_2)=(0.9,0.99)$, optimizer $\epsilon=10^{-8}$, zero weight decay, and gradient clipping at 1.0.
The selected weight, activation, or gradient signal is round-tripped through FP16 or emulated FP8 E4M3FN, while arithmetic, gradient accumulation, and optimizer state remain FP32.
Activation conversions use a straight-through derivative, and gradient conversions are applied both to the error entering leaf-module backward computations and to the accumulated parameter gradient supplied to AdamW.
The FP8 weight-access condition uses FP8 read values from an FP16-rounded master weight that is restored before the AdamW update, so it tests low-precision weight access without replacing the optimizer's master state.
For each seed, the final validation cross-entropy is averaged over 200 mini-batches.
Each bar is the mean of these four seed-level values and its whiskers give their observed minimum and maximum, rather than a confidence interval.
Thus, the control isolates value-precision sensitivity and does not claim the use of native FP8 training kernels.

\paragraph{(b) Gradient cancellation by projection type.}
We freeze a step-5,000 digital AdamW checkpoint and, for each of four sampling seeds (1--4), draw eight batches of 64 length-256 sequences from the Shakespeare training split.
The model remains in training mode with dropout 0.2; dropout is seeded separately for each batch, and no optimizer step, gradient clipping, or autocast is applied.
For each linear projection $Y=XW^{\mathsf T}$, we capture its FP32 input $X$ and output gradient $D=\partial\mathcal L/\partial Y$, flattening the batch and token dimensions into $N=16{,}384$ positions.
The loss uses mean cross-entropy, so its normalization is already included in $D$.
We form the reference products and reductions in FP64:
\begin{equation}
g_{n,ij}=D_{n,i}X_{n,j},\qquad
G_{ij}=\sum_{n=1}^{N}g_{n,ij},\qquad
A_{ij}=\sum_{n=1}^{N}|g_{n,ij}|,\qquad
\kappa_{ij}=\frac{A_{ij}}{|G_{ij}|}.
\label{eq:gradient-cancellation-diagnostic}
\end{equation}
Here, $n$ indexes a local linear-layer contribution, not the parameter gradient of an isolated token loss, since attention couples loss positions.
The six bars correspond to query, key, value, attention output, and MLP up/down projections, with the fused QKV output split into three row groups.
For each seed, the statistic takes the arithmetic mean of $\kappa$ over weights and batches within a block, then averages equally over the eight blocks, without a logarithmic transformation or clipping.
Each bar is the arithmetic mean over the four seed-level statistics and its whiskers give their observed minimum and maximum, rather than a confidence interval.
These arithmetic means are sensitive to coordinates with nearly complete cancellation and should not be interpreted as typical coordinate-level values.
We exclude embeddings, the tied output head, layer-normalization parameters, and the attention-score and attention-value products.
All-zero contributions and exact cancellation are counted separately; neither occurs in the measured projections.
The reconstructed gradients agree with autograd to a maximum relative Frobenius error of $3.54\times10^{-7}$.

\paragraph{(c) Controlled accumulation-precision replay.}
For each seed, we sample 64 weight coordinates without replacement from each of the 48 projections and reuse these coordinates across that seed's eight batches, giving 24,576 weight--batch observations per seed.
For each coordinate, the same $N$ FP64 contributions are replayed in batch-major, token-major order, rounding the accumulator after every addition in FP32, FP16, or FP8 E4M3FN while retaining FP64 addends.
For each projection and batch, a fixed power-of-two scale places the largest sampled addend magnitude in $[0.5,1)$, with the same scale used for all coordinates and formats in that group.
The format conversions retain their finite exponent range and underflow behavior.
We restore the original units before comparison with the FP64 reference sum.
Observations are grouped into bins of width 0.5 in $\log_{10}\kappa$, and each seed-level error is $\|\widehat{\mathbf G}-\mathbf G\|_2/\|\mathbf G\|_2$ over the coordinates in that bin; bins containing fewer than 30 observations are omitted.
The lines average these bin-level errors over the four seeds and the shaded bands show their observed seed minimum and maximum, rather than confidence intervals.
This experiment isolates sequential accumulation error for fixed captured signals; it does not reproduce the reduction tree of a hardware matrix-multiplication kernel or establish end-to-end training failure.

\paragraph{(d) Micro-batch update error.}
The one-step diagnostic uses a bias-free $144\rightarrow32$ linear layer and, for each of four seeds (1--4), a fixed batch of 128 Gaussian input/output-gradient pairs.
The initial weights have zero prescribed mean and RMS 0.12, and the output gradients are rescaled so that the full-batch target update at learning rate 0.1 has RMS 0.012.
A single symmetric soft-bounds tile has conductance bounds $\pm0.8$ and nominal write increment $\Delta w_{\min}=0.8/600$, with device and cycle variation disabled, perfect forward/backward reads, deterministic implicit pulses, update management, and a 31-pulse limit.
Each condition starts from the same weights and ordered sample pairs, using micro-batch sizes 128, 32, 8, 2, or 1.
The MP path accumulates each micro-batch update digitally and immediately transfers truncated pulse quanta, capped at $\pm31$ per entry, retaining the untransferred remainder in a digital residual.
The rank-update path applies the sample outer products directly through the tile update operation.
The vertical axis is $\|\Delta W_{\mathrm{realized}}-\Delta W_{\mathrm{target}}\|_F/\|\Delta W_{\mathrm{target}}\|_F$ after processing the fixed batch, with no additional residual flush.
Each bar is the mean across the four seed-level errors and its whiskers give their observed minimum and maximum, rather than a confidence interval.
The red bar uses the rank-update error at micro-batch size 128 from each seed.
Because this diagnostic is symmetric and single-step, it isolates error accumulated through repeated quantized transfers rather than asymmetric device dynamics.
This controlled one-step diagnostic isolates the write-path consequence of micro-batch splitting; it is not an end-to-end network-training experiment.
The direct rank-update produces a large but bitwise identical realized update for every tested partition in this deterministic, near-linear symmetric-device configuration, so panel (d) displays one representative red bar.
This partition invariance is not a claim about stochastic, asymmetric, saturated, or multi-step training trajectories.

\subsection{Initialization maps (Figure~\ref{fig:initial-native-versus-s-scaling})}
\label{app:initialization-map-details}

This section specifies the illustrative width-48 non-residual initialization distributions in Figure~\ref{fig:initial-native-versus-s-scaling}.
They are visual comparisons of the stated mappings, not pooled histograms across the transformer or exact reproductions of every training-system choice in the cited work.

\paragraph{Rasch-inspired native-AbsMax control.}
We first draw a logical matrix $W_0$ with entries of target logical-space initialization standard deviation $\sigma_w=0.08$.
The plotted native virtual mapping uses $\omega_{\mathrm{map}}=0.3$, defines $a=\lVert W_0\rVert_\infty/(\omega_{\mathrm{map}}\tau)$, and stores $\widehat W=W_0/a$.
Thus, the realized maximum magnitude maps to $\omega_{\mathrm{map}}\tau$, and the compensating logical scale $a$ is matrix-specific.
This is the fixed-AbsMax virtual-mapping convention used in AIHWKit-style Rasch controls~\citep{rasch2020large,rasch2021aihwkit}, not a literal reproduction of the direct rank-one-SGD RPU training algorithm.

\paragraph{Nandakumar-inspired fixed-logical-range control.}
\cite{nandakumar2020mixed} reports PCM conductances initialized with mean $4.5\,\mu\mathrm{S}$ and standard deviation $1.25\,\mu\mathrm{S}$ over the device range $[0.1,8]\,\mu\mathrm{S}$.
For a shared display coordinate only, we transform these conductances as
\begin{equation}
\widetilde G = \frac{2(G-0.1\,\mu\mathrm{S})}{(8-0.1)\,\mu\mathrm{S}}-1 \in [-1,1].
\label{eq:nandakumar-display-normalization}
\end{equation}
The figure draws a representative sample from the reported initialization distribution, clipped to the stated device interval, then applies \eqref{eq:nandakumar-display-normalization}.
In their first epoch, the same conductance range maps to logical weights in $[-0.7,0.7]$, which the figure represents as $W^{(1)}=0.7\widetilde G$.
This PCM curve is therefore an illustrative normalization of the reported hardware initialization, not the S-NW ablation in Table~\ref{tab:profile-ablation-screen}; S-NW retains S's physical initialization and changes only the read-scale interpretation during its first two nominal epochs.

\subsection{Input-DAC reconstruction-distortion experiment (Figure~\ref{fig:normalization-mmse})}
\label{app:normalization-mmse-calibration}

This section gives the complete configuration and interpretation for Figure~\ref{fig:normalization-mmse}; it is a controlled input-quantization calculation with neither an MVM nor device-read noise.
For the precision scans, we draw $N=100{,}000$ vectors of width $D=48$ from each unit-variance product law, $x_i\sim\mathcal N(0,1)$ and $x_i\sim\operatorname{Laplace}(0,1/\sqrt 2)$, using pseudorandom seed $20{,}260{,}823$ and double-precision arithmetic.
The first two panels sweep $b=2,\ldots,10$ with $K_b=2^b-2$ quantization intervals.
The two right panels fix $b=6$ and sweep $D=2$ through $32{,}768$; each point uses $\max\{2{,}000,\min\{20{,}000,\lceil8{,}000{,}000/D\rceil\}\}$ vectors.

For calibration vectors $x^{(1)},\ldots,x^{(N)}$ and a set $\mathcal H_b\subset[1,\infty]\times(0,\infty)$ of supported norm--rail pairs, the selected rail is
\begin{equation}
(p_b,A_b)\in\operatorname*{arg\,min}_{(p,A)\in\mathcal H_b}
\frac1N\sum_{n=1}^N
\left\lVert Q_{K_b,A\lVert x^{(n)}\rVert_p}\bigl(x^{(n)}\bigr)-x^{(n)}\right\rVert_2^2,
\qquad
s_b(x)=A_b\lVert x\rVert_{p_b}.
\label{eq:exact-calibrated-range}
\end{equation}
Here, $K_b+1$ is the number of reconstruction levels.
The selected pair remains fixed after calibration, whereas $s_b(x)$ adapts to each input vector.
The candidate set includes $(p,A)=(\infty,1)$, so the calibration explicitly includes the no-clipping AbsMax rail.

The vertical coordinate is the finite-sample estimate of the average reconstruction distortion in \eqref{eq:joint-norm-rail-distortion},
\begin{equation}
\widehat R_b(q,a)
=\frac1{ND}\sum_{n=1}^N \rho_q(x^{(n)})^2
\left\lVert Q_{K_b,a}\!\left(\frac{x^{(n)}}{\rho_q(x^{(n)})}\right)-\frac{x^{(n)}}{\rho_q(x^{(n)})}\right\rVert_2^2,
\label{eq:normalization-calibration-empirical-distortion}
\end{equation}
where $q=1/p$ and
\begin{equation}
\rho_q(x)=
\begin{cases}
\left(D^{-1}\sum_{i=1}^D |x_i|^{1/q}\right)^q,&q>0,\\
\lVert x\rVert_\infty,&q=0.
\end{cases}
\label{eq:normalization-calibration-power-mean}
\end{equation}
Thus the plotted quantity estimates $R_K$ in the original input units; it is not the input-energy-normalized NMSE used in the preceding version of the figure.
For $q>0$, the corresponding norm--rail parameters are $p=1/q$ and $A=aD^{-q}$, with $p=\infty$ and $A=a$ at $q=0$.

The search starts from 17 equally spaced $q$ values on $[0,1]$.
For each precision, five adaptive midpoint-refinement rounds and one interior parabolic proposal refine the best sampled norm order.
At each sampled $q$, 401 rails on $0.02\leq a\leq1.25D^q$ are evaluated, and the neighboring interval around the best rail is refined three times with 101 points.
For each candidate, weighted prefix sums evaluate the deterministic round-to-even saturating quantizer directly, so both granular and clipping error are included.
The returned candidate is requantized in logical units as a numerical check.
The displayed fixed-RMS reference uses $q=1/2$ and $a=3$ without optimization; the one- and two-RMS references are omitted for visual economy.

Each panel also evaluates two static ACIQ thresholds: one from ACIQ's unit-variance Gaussian approximation and one from its unit-variance Laplace approximation~\cite{banner2019post}.
The legend identifies the prior used to derive the fixed threshold, while the panel title identifies the distribution on which that threshold is evaluated.
For the unit-variance Laplace prior, the threshold minimizes $2b_L^2\exp(-\alpha/b_L)+\alpha^2/(3\,2^{2b})$ with $b_L=1/\sqrt2$.
For the unit-variance Gaussian prior, it numerically minimizes $(\alpha^2+1)\operatorname{erfc}(\alpha/\sqrt2)+\alpha^2/(3\,2^{2b})-\sqrt{2/\pi}\,\alpha\exp(-\alpha^2/2)$.
These approximations use ACIQ's $2^b$ mid-rise bins, whereas all displayed curves directly evaluate the resulting threshold with the manuscript's endpoint-including, round-to-even quantizer with $K_b=2^b-2$.

At $D=48$, AbsMax is within 10\% of the calibrated norm--rail distortion from four bits onward for both product laws, and they nearly coincide at higher precision.
At six bits, the same 10\% criterion holds through $D=1{,}024$; the curves separate from about $D=2{,}048$, especially for Laplace inputs.
The fixed ACIQ thresholds can reduce low-bit error when their assumed prior matches the input law, but their clipping cost becomes visible at higher precision and under a mismatched prior.
The ACIQ curves are nearly independent of $D$ because each applies one coordinatewise, distribution-level threshold.
The largest width exceeds the $d_{\mathrm{model}}=12{,}288$ of the 175B-parameter dense GPT-3 reference~\cite{brown2020language}, although width alone does not determine parameter count or end-to-end AIMC-training impact.
Because the inputs are independent product distributions rather than measured training traces, this experiment validates the distributional quantization calculation rather than the representativeness of either law for every model layer.

\subsection{Additional ADC Range-Alignment Designs}
\label{sec:adc-range-alignment}

The near-optimal performance of max-aligned DAC rails motivates an analogous ADC principle: avoid clipping the largest output coordinate.
Unlike DACs, however, the ADC receives the already-formed analog accumulator $z=\widehat W x+\eta$, so its range cannot be enlarged retrospectively.
To avoid clipping, we adopt \textit{bound management (BM)} proposed by \cite{gokmen2017cnn}.
When the maximum output coordinate clips on a first read, BM repeatly reduces the input from $x$ to $x/2$ until no clipped coordinate remains after the ADC read.
This is the BM design at the right of Figure~\ref{fig:overview-comparison}; it avoids irreversible maximum-coordinate clipping, but every retry repeats the array MVM and its peripheral conversions.
Consequently, the ADC range $C_{\mathrm{ADC}}$ should therefore balance quantization resolution against retries.
If $C_{\mathrm{ADC}}$ is too large, typical outputs occupy only a small fraction of the converter range and waste quantization levels.
If it is too small, BM is triggered repeatedly and increases the computation required for one logical MVM.

\begin{wrapfigure}[13]{r}{0.40\textwidth}
\centering
\vspace{-2.5em}
\includegraphics[width=\linewidth]{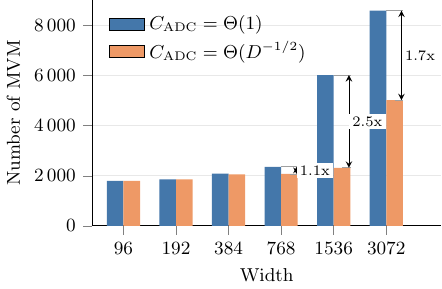}
\vspace{-2em}
\caption{Reduction of bound-management MVM retries through output-rail alignment.}
\label{fig:forward-rail-large-width-reads}
\end{wrapfigure}
The conductance distribution designed in Section~\ref{sec:stable-weight-scaling} makes this tradeoff predictable.
By initialization, the target conductance standard deviation is $\tau/\omega$, and the input has standard deviation $\sigma_x$, so the standard deviation of each output coordinate is approximately $(\tau/\omega)\sqrt{D}\,\sigma_x$.
It motivates us to choose
\begin{equation}
C_{\mathrm{ADC}}
=c_{\mathrm{out}}(\tau/\omega)\sqrt{D}\,\sigma_x,
\label{eq:predicted-adc-rail}
\end{equation}
where $c_{\mathrm{out}}$ is a calibrated tail multiplier.
In S-OUT, native AbsMax noise management first rescales each input vector for the tile read and restores that scale after the ADC; the implemented tile rail is the static dimension-dependent value $c_{\mathrm{out}}(\tau/\omega)\sqrt{D}$.
The calibrated $c_{\mathrm{out}}$ accounts for the distribution of the managed input scale rather than evaluating $\sigma_x$ separately for each vector.
Figure~\ref{fig:forward-rail-large-width-reads} provides empirical evidence showing that this choice reduces the number of MVM retries, especially at larger widths.
To isolate the operational effect of the output rail at larger widths, Figure~\ref{fig:forward-rail-large-width-reads} retains literal native S for the language-model head.
The OUT condition changes only non-head forward MVMs, applying $B_{\mathrm{out}}=6\sqrt{D_{\mathrm{in}}}/3$ while retaining iterative bound management; this prevents the unusually heavy head-output tail from conflating a static-rail comparison for ordinary Transformer MVMs with a separate head-range intervention.
Width 96 uses 5,000 updates; widths 192, 384, and 768 continue their 5,000-update checkpoints to 10,000 updates at the original minimum learning rate of $10^{-4}$, identically for S and OUT.
Widths 1536 and 3072 continue their original checkpoints to 20,000 and 40,000 updates, respectively, retaining their original decay endpoints of 10,000 and 20,000 updates and minimum learning rate of $10^{-4}$.
The width-192--768 continuation retains its original learning-rate decay endpoint at 5,000 updates; all continued conditions have their read counts recomputed from the extended checkpoints.
All cells use seed 1337.

\subsection{Scaling Analog Training (Figure~\ref{fig:scaling} and Table~\ref{tab:scaling-model-specifications})}
\label{app:simulation-details-scaling}

We compared five decoder-only GPT models with tied token-embedding and output-projection weights, no bias terms, a 50,257-token vocabulary, and a context length of 1,024 tokens.
Table~\ref{tab:scaling-model-specifications} lists the model specifications and schedules.
\begin{table}[t]
\centering
\caption{Decoder-only GPT model specifications and training schedules for the scaling-law simulations.}
\label{tab:scaling-model-specifications}
\small
\setlength{\tabcolsep}{3pt}
\begin{tabular}{lrrrrrr}
\toprule
Model & \# blocks & Width & \# heads & \# parameters & Updates & Final validation step \\
\midrule
Tiny   & 2  & 128 & 4  & 6.83M   & 278   & 273 \\
Small  & 4  & 256 & 4  & 16.01M  & 652   & 640 \\
Medium & 6  & 384 & 6  & 29.92M  & 1,218 & 1,200 \\
Large  & 8  & 512 & 8  & 50.91M  & 2,072 & 2,060 \\
XLarge & 12 & 768 & 12 & 123.55M & 5,028 & 5,020 \\
\bottomrule
\end{tabular}
\end{table}
All models were trained from random initialization on the tokenized OpenWebText training split and evaluated on its held-out validation split.
Writing the model parameter count as $N$, the planned token budget at each scale was approximately $\#\mathrm{data}=20N$, following the Chinchilla allocation rule~\citep{hoffmann2022training}.
Each optimizer update used a global batch of 480 sequences of length 1,024, or 491,520 tokens.
Both paths used linear warm-up followed by cosine learning-rate decay from a peak of $6\times10^{-4}$ to $6\times10^{-5}$, AdamW coefficients $\beta_1=0.9$ and $\beta_2=0.95$, zero weight decay, and gradient clipping at 1.0.

Each validation loss averages ten evaluation batches.
PyTorch was seeded with 1337, and each condition was run once; differences between curves therefore describe these runs rather than run-to-run variability.

The digital baseline used PyTorch AdamW with BF16 autocasting.
The analog path used the complete S-OUT algorithm: Li-ECRAM tiles, 16-bit DACs and 9-bit ADCs in both forward and backward directions, native input management, and matrix-specific ADC output rails with $c_{\mathrm{out}}=6.0$.
The analog path computed digital W-grad, while its logical-space AdamW optimizer maintained moment estimates and retained sub-threshold update residuals digitally before issuing bounded, quantized pulse updates to the tiles.
The analog runs used float32 because this AIHWKit training path did not support BF16 autocasting.
Dense linear maps, including the tied vocabulary projection, were mapped to analog tiles without array-size tiling limits; attention-score products, softmax, normalization, and other nonlinear operations remained digital.
The tile updates used fixed bit-line length, update and bit-line management, and deterministic implicit pulses; no Hadamard preconditioning was enabled.

For both paths, we denote the full-model algorithmic compute proxy by $\#\mathrm{compute}$.
It is $\#\mathrm{data}$ times the sum of $6\,\#\mathrm{nonemb}$, $6\,\#\mathrm{vocab}\,\#\mathrm{width}$, and $12\,\#\mathrm{layers}\,\#\mathrm{width}\,\#\mathrm{context}$.
Here, $\#\mathrm{nonemb}$ is the number of non-embedding parameters, $\#\mathrm{vocab}$ is the vocabulary size, $\#\mathrm{width}$ is the embedding width, $\#\mathrm{layers}$ is the number of Transformer blocks, and $\#\mathrm{context}$ is the context length.
This proxy includes the vocabulary head and attention while excluding the extra work of the simulator's dense one-hot embedding implementation.
Figures~\ref{fig:scaling} show the final recorded evaluations and observed trajectories of the five completed matched schedules.
For each training path and each coordinate $X\in\{\#\mathrm{compute},\#\mathrm{data},N\}$ in Figure~\ref{fig:scaling}, we fit $L=A X^{-\alpha}$ by ordinary least squares in $\log L$--$\log X$ space using its five completed endpoints.
Solid curves show the fitted relationships over the observed coordinate ranges; dashed curves extend them to ten times the largest observed coordinate.
These two-parameter fits omit an irreducible-loss term and serve as preliminary guides for selecting the next model scale.
Because $\#\mathrm{compute}$, $\#\mathrm{data}$, and $N$ covary across single-seed schedules, the fits cannot isolate separate parameter and data exponents or establish an asymptotic scaling law.

\subsection{Zero-shot downstream evaluation}
\label{app:simulation-details-downstream}

We evaluated the final Tiny-through-Xlarge checkpoints from digital AdamW and analog training on 22 multiple-choice or verbalized-label tasks.
The latter are separately trained hardware-aware models whose dense layers execute through the simulated EcRam tiles during inference; they are not direct-programmed versions of the digital checkpoints.
All models used the GPT-2 byte-pair tokenizer, a 50,257-token vocabulary and a maximum context of 1,024 tokens.
For each question, we selected the answer continuation with the largest mean token log-likelihood conditional on the task prompt, following the conditional-likelihood formulation used for zero-shot language-model task transfer.
If a prompt and candidate exceeded the context window, we retained the final 1,024 tokens so that the complete answer continuation was scored.
The table reports accuracy under this length-normalized criterion; no gradient computation, calibration or downstream fine-tuning was performed.

The commonsense and multiple-choice reasoning suite comprises HellaSwag, PIQA, ARC-Easy, ARC-Challenge, OpenBookQA, WinoGrande, CommonsenseQA, COPA, SciQ and WSC.
The complementary language-understanding suite comprises AG News, BoolQ, CB, CoLA, MNLI, MRPC, QNLI, QQP, RTE, SST-2, WiC and WNLI.
We used the labelled validation split for WinoGrande, GLUE and SuperGLUE tasks, the test split for ARC, OpenBookQA, AG News and SciQ, and the validation split for HellaSwag, PIQA and CommonsenseQA.
For evaluation sets with at least 128 labelled examples, we used a fixed random sample of 128 examples with seed 1337; smaller sets were evaluated in full.
The tasks were expressed through simple task-specific natural-language prompts and answer-label continuations, with the same prompt and example selection for every compared checkpoint.

\begin{table*}[t]
\centering
\footnotesize
\setlength{\tabcolsep}{2.5pt}

\begin{tabular}{lrrrrrrrrrr@{\hspace{6pt}}r}
\toprule
& \multicolumn{10}{c}{Task accuracy $\uparrow$} & \multirow{2}{*}{\textbf{Avg} $\uparrow$} \\
\cmidrule(lr){2-11}
Model & HSw & PIQA & ARCe & ARCc & OBQA & Wino & CSQA & COPA & SciQ & WSC & \\
\midrule
Digital Tiny    & 29.7 & 57.0 & 29.7 & 23.4 & 21.9 & 47.7 & 18.8 & 53.0 & 35.9 & 36.5 & \textbf{35.4} \\
Digital Small   & 28.9 & 61.7 & 33.6 & 22.7 & 17.2 & 50.0 & 21.1 & 49.0 & 43.0 & 59.6 & \textbf{38.7} \\
Digital Medium  & 27.3 & 58.6 & 33.6 & 18.8 & 21.9 & 49.2 & 27.3 & 52.0 & 48.4 & 60.6 & \textbf{39.8} \\
Digital Large   & 26.6 & 60.2 & 33.6 & 21.1 & 18.0 & 50.8 & 25.0 & 54.0 & 39.8 & 63.5 & \textbf{39.3} \\
Digital Xlarge  & 28.9 & 64.1 & 35.9 & 21.1 & 21.1 & 50.8 & 21.9 & 59.0 & 46.1 & 41.3 & \textbf{39.0} \\
\midrule
Analog Tiny   & 26.6 & 56.2 & 27.3 & 21.1 & 21.1 & 48.4 & 22.7 & 47.0 & 37.5 & 62.5 & \textbf{37.0} \\
Analog Small  & 25.0 & 57.0 & 25.8 & 18.8 & 21.9 & 50.8 & 18.8 & 52.0 & 38.3 & 39.4 & \textbf{34.8} \\
Analog Medium & 22.7 & 56.2 & 28.9 & 20.3 & 21.1 & 48.4 & 21.1 & 51.0 & 38.3 & 45.2 & \textbf{35.3} \\
Analog Large  & 30.5 & 59.4 & 29.7 & 20.3 & 21.9 & 41.4 & 22.7 & 49.0 & 38.3 & 55.8 & \textbf{36.9} \\
Analog Xlarge & 25.8 & 55.5 & 36.7 & 25.8 & 17.2 & 45.3 & 26.6 & 55.0 & 43.8 & 44.2 & \textbf{37.6} \\
\bottomrule
\end{tabular}
\caption{Zero-shot commonsense and multiple-choice reasoning screening accuracy (\%, length-normalized conditional likelihood).
Digital AdamW checkpoints appear first, followed by separately trained analog checkpoints.
Each task uses 128 fixed-seed labelled examples where available, and otherwise its complete labelled evaluation split (COPA: 100 and WSC: 104 examples).
HSw: HellaSwag; ARCe/ARCc: ARC-Easy/ARC-Challenge; OBQA: OpenBookQA; Wino: WinoGrande; CSQA: CommonsenseQA.
Avg is the unweighted average across the ten displayed tasks.}
\label{tab:downstream-zero-shot}
\end{table*}

\begin{table*}[t]
\centering
\footnotesize
\setlength{\tabcolsep}{2.5pt}

\begin{tabular}{l*{12}{r}r}
\toprule
& \multicolumn{12}{c}{Task accuracy $\uparrow$} & \multirow{2}{*}{\textbf{Avg} $\uparrow$} \\
\cmidrule(lr){2-13}
Model & AG & Bool & CB & CoLA & MNLI & MRPC & QNLI & QQP & RTE & SST-2 & WiC & WNLI & \\
\midrule
Digital Tiny    & 18.8 & 39.8 & 41.1 & 46.1 & 35.9 & 39.1 & 47.7 & 51.6 & 52.3 & 58.6 & 49.2 & 56.3 & \textbf{44.7} \\
Digital Small   & 21.1 & 42.2 & 41.1 & 43.0 & 38.3 & 35.9 & 46.1 & 39.8 & 52.3 & 57.0 & 48.4 & 62.0 & \textbf{43.9} \\
Digital Medium  & 21.9 & 40.6 & 41.1 & 41.4 & 38.3 & 32.8 & 56.2 & 68.0 & 51.6 & 64.8 & 46.1 & 59.2 & \textbf{46.8} \\
Digital Large   & 25.0 & 40.6 & 41.1 & 40.6 & 40.6 & 32.8 & 47.7 & 65.6 & 51.6 & 60.9 & 46.1 & 56.3 & \textbf{45.7} \\
Digital Xlarge  & 22.7 & 56.2 & 41.1 & 39.1 & 38.3 & 33.6 & 58.6 & 60.2 & 52.3 & 57.0 & 46.1 & 63.4 & \textbf{47.4} \\
\midrule
Analog S Tiny   & 20.3 & 40.6 & 41.1 & 39.1 & 38.3 & 32.8 & 47.7 & 65.6 & 52.3 & 53.9 & 46.9 & 57.7 & \textbf{44.7} \\
Analog S Small  & 20.3 & 49.2 & 41.1 & 50.0 & 38.3 & 42.2 & 50.0 & 44.5 & 56.2 & 50.8 & 49.2 & 59.2 & \textbf{45.9} \\
Analog S Medium & 24.2 & 45.3 & 41.1 & 50.0 & 38.3 & 37.5 & 57.8 & 45.3 & 51.6 & 53.1 & 48.4 & 50.7 & \textbf{45.3} \\
Analog S Large  & 19.5 & 36.7 & 41.1 & 38.3 & 38.3 & 34.4 & 50.8 & 64.1 & 50.0 & 48.4 & 46.9 & 53.5 & \textbf{43.5} \\
Analog S Xlarge & 32.8 & 56.2 & 41.1 & 41.4 & 39.1 & 32.8 & 57.8 & 64.8 & 53.9 & 52.3 & 46.1 & 53.5 & \textbf{47.7} \\
\bottomrule
\end{tabular}
\caption{Complementary zero-shot language-understanding screening accuracy (\%, length-normalized conditional likelihood). Digital AdamW checkpoints appear first, followed by separately trained analog checkpoints.
Each task uses 128 fixed-seed labelled examples where available, and otherwise its complete labelled evaluation split (CB: 56 and WNLI: 71 examples). Avg is the unweighted average across the 12 displayed tasks.}
\label{tab:downstream-language-understanding}
\end{table*}

The digital and analog models show task- and scale-dependent macro-average screening accuracy, with variation in either direction across individual tasks.
This comparison is descriptive: it uses one realization per training condition and fixed evaluation samples, and several tasks contain fewer than 128 labelled validation examples.
The scaling-law pre-training losses nevertheless remain lower for the digital checkpoints (Figure~\ref{fig:scaling}), so the downstream screen does not remove the observed optimization gap.

This protocol is intended as a controlled transfer screen rather than a leaderboard evaluation.
In particular, it reports accuracy for all tasks, whereas several GLUE and SuperGLUE benchmarks conventionally use other primary metrics, and it does not reproduce every benchmark's canonical prompt or calibration procedure.
Single device realizations and the fixed subsets mean that task-level differences, especially on small evaluation sets, should not be interpreted as estimates of statistical significance.

\subsection{Algorithm robustness across general RPU device models}
\label{app:simulation-details-device-comparison}

We test the algorithms through AIHWKit's general \texttt{SingleRPUConfig} interface, so the device update rule is an experimental axis rather than an unexamined simulator default.
Tables~\ref{tab:aihwkit-real-device-rpu-configs-a}--\ref{tab:aihwkit-real-device-rpu-configs-b} list every AIHWKit preset whose documentation identifies a measured physical device or array; the idealized and Gokmen--Vlasov algorithmic reference devices are excluded.
Each row cites the measurement paper used for the fit, and a dash denotes a parameter that the corresponding class does not define.
The AIHWKit implementation is cited with its officially requested toolkit reference~\cite{buechel2024aihwkit}.
\emph{ReRamES} and \emph{ReRamSB} denote, respectively, exponential-step and soft-bounds fits to the HfO$_2$ ReRAM signal-and-noise measurements of \cite{gong2022reram}.
\emph{Capacitor} denotes the capacitor cross-point array fit of \cite{li2018capacitor}; \emph{Li-ECRAM} and \emph{MO-ECRAM} denote lithium-ion and metal-oxide electrochemical RAM fits from \cite{tang2018ecram} and Kim et al.~\cite{kim2019moecram}, respectively.
\emph{PCM} denotes the Ge$_2$Sb$_2$Te$_5$ phase-change-memory model used by Nandakumar et al.~\cite{nandakumar2019phase}.
\emph{ReRamArrayOM} and \emph{ReRamArrayHfO$_2$} denote, respectively, the optimized-material and baseline-HfO$_2$ ReRAM-array fits reported by \cite{gong2022reram}.

\begin{table*}[t]
\centering
\small
\setlength{\tabcolsep}{2.5pt}
\begin{tabular}{lrrrrrrr}
\toprule
Preset device & $\mathrm{dw}_{\min}$ & $\mathrm{up\_down}$ & $w_{\max}$ & $w_{\min}$ & $\gamma_{\mathrm{up}}$ & $\gamma_{\mathrm{down}}$ & \# states \\
\midrule
ReRamES~\cite{gong2022reram} & 0.00135 & 0.259359 & 1 & -1 & 5 & 5 & 1481.5 \\
ReRamSB~\cite{gong2022reram} & 0.002 & 0 & 1.25 & -0.75 & -- & -- & 1000 \\
Capacitor~\cite{li2018capacitor} & 0.005 & 0 & 1 & -1 & 0.05 & 0.05 & 400 \\
Li-ECRAM~\cite{tang2018ecram} & 0.002 & 0 & 1.1724 & -0.8276 & 0.1153 & 0.5085 & 1000 \\
MO-ECRAM~\cite{kim2019moecram} & 0.00028214 & 0 & 1.1714 & -0.8286 & 0.4152 & 0.7342 & 7088.7 \\
PCM~\cite{nandakumar2019phase} & 0.01 & 0 & 2 & 0 & 2.5 & 2.5 & 200 \\
ReRamArrayOM~\cite{gong2022reram} & 0.0949 & 0 & 1 & -1 & -- & -- & 21.1 \\
ReRamArrayHfO$_2$~\cite{gong2022reram} & 0.4622 & 0 & 1 & -1 & -- & -- & 4.3 \\
\bottomrule
\end{tabular}
\caption{AIHWKit \texttt{SingleRPUConfig} presets fitted to measured physical devices or arrays: primary update and conductance parameters. \# states denotes the nominal ratio $(w_{\max}-w_{\min})/\mathrm{dw}_{\min}$. The citation in each row identifies the source measurement.}
\label{tab:aihwkit-real-device-rpu-configs-a}
\end{table*}

\begin{table*}[t]
\centering
\scriptsize
\resizebox{\textwidth}{!}{%
\begin{tabular}{lrrrrrrrr}
\toprule
Preset device & $\mathrm{dw}_{\min,\mathrm{d2d}}$ & $\mathrm{up\_down}_{\mathrm{d2d}}$ & $w_{\max,\mathrm{d2d}}$ & $w_{\min,\mathrm{d2d}}$ & $\gamma_{\mathrm{up,d2d}}$ & $\gamma_{\mathrm{down,d2d}}$ & $\mathrm{dw}_{\min,\mathrm{std}}$ & $\mathrm{write\_noise}_{\mathrm{std}}$ \\
\midrule
ReRamES~\cite{gong2022reram} & 0.2 & 0.05 & 0.3 & 0.3 & -- & -- & 5 & 75 \\
ReRamSB~\cite{gong2022reram} & 0.3 & 0.01 & 0.24 & 0.4 & -- & -- & 3.75 & 56 \\
Capacitor~\cite{li2018capacitor} & 0.1 & 0.06 & 0.07 & 0.07 & 0.01 & 0.01 & 0.3 & 0 \\
Li-ECRAM~\cite{tang2018ecram} & 0.1 & 0.01 & 0.05 & 0.05 & 0.05 & 0.05 & 0.3 & 0 \\
MO-ECRAM~\cite{kim2019moecram} & 0.1 & 0.01 & 0.05 & 0.05 & 0.05 & 0.05 & 2 & 0 \\
PCM~\cite{nandakumar2019phase} & 0.2 & 0.05 & 0.1 & 0 & -- & -- & 0.6 & 0 \\
ReRamArrayOM~\cite{gong2022reram} & 0.7829 & 0.01 & 0.3499 & 0.5695 & -- & -- & 0.4158 & 1.4113 \\
ReRamArrayHfO$_2$~\cite{gong2022reram} & 0.7125 & 0.01 & 0.4295 & 0.599 & -- & -- & 0.2174 & 0.5841 \\
\bottomrule
\end{tabular}}
\caption{AIHWKit \texttt{SingleRPUConfig} presets fitted to measured physical devices or arrays: device and cycle variation parameters. The citation in each row identifies the source measurement.}
\label{tab:aihwkit-real-device-rpu-configs-b}
\end{table*}

Figure~\ref{fig:measured-rpu-update-response} visualizes five stochastic potentiation--depression trajectories for every measured preset.
The curves retain the preset's cycle and device variation and therefore illustrate the modelled update behavior, not an idealized mean response.

\begin{figure*}[t]
\centering
\includegraphics[width=\textwidth]{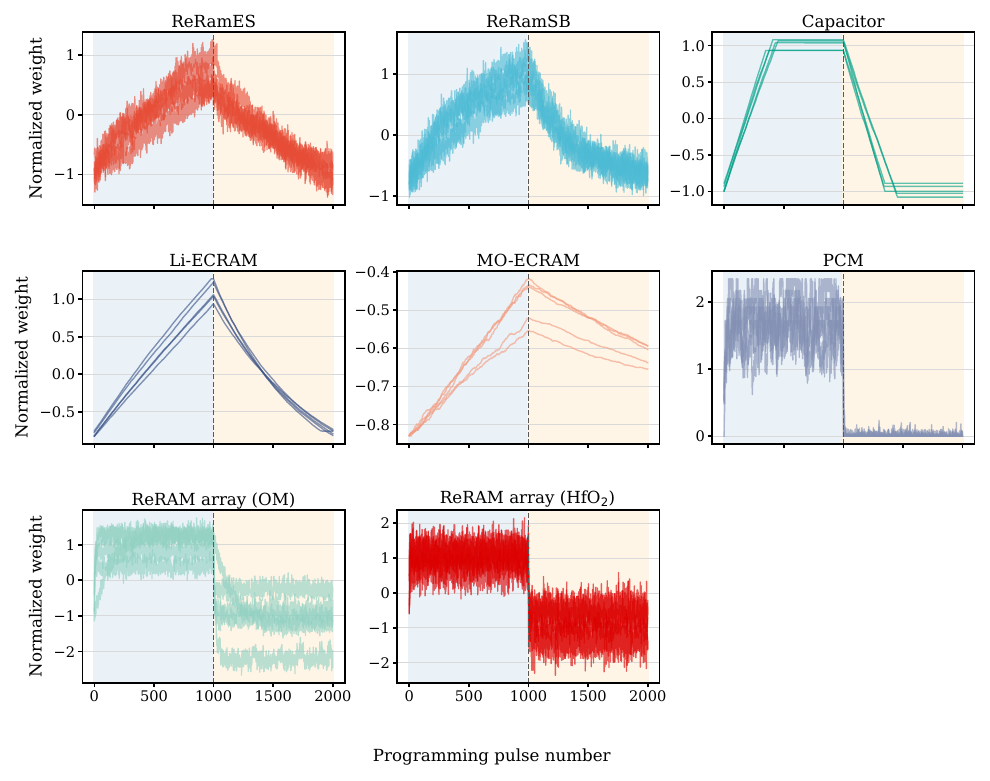}
\caption{Response curves for all measured-device \texttt{SingleRPUConfig} presets in Tables~\ref{tab:aihwkit-real-device-rpu-configs-a}--\ref{tab:aihwkit-real-device-rpu-configs-b}. Each panel shows five simulated traces under 1,000 potentiation pulses (pale blue) followed by 1,000 depression pulses (pale orange). The dashed line marks the reversal.}
\label{fig:measured-rpu-update-response}
\end{figure*}

For the algorithm-performance comparison, we evaluate the calibrated SoftBounds control and all eight measured-device RPU models documented above: ReRamES, ReRamSB, capacitor, Li-ECRAM, MO-ECRAM, PCM, oxide-memory (OM) ReRAM, and HfO$_2$ ReRAM.
We compare methods S and SFB on the Shakespeare character-level task with one attention head, embedding width 48, dropout 0.2, and 2, 4, 6, or 8 transformer blocks; a matched Digital run is included as a non-analog baseline.
Each condition runs for 5,000 iterations and is evaluated every 250 iterations over 200 validation batches at seed 1337.
The results are therefore descriptive paired comparisons rather than estimates of seed-to-seed variation.

\begin{table}[H]
\centering
\small

\begin{tabular}{lrrrr}
\toprule
Device & L2 & L4 & L6 & L8 \\
\midrule
Digital & 2.2276 & 2.0239 & 1.9006 & 1.8422 \\
\midrule
SoftBounds & 2.3302 & 2.0935 & 1.9733 & 1.8949 \\
ReRamES & 2.4486 & 2.3741 & 2.2386 & 2.1687 \\
ReRamSB & 2.4782 & 2.4455 & 2.3214 & 2.2486 \\
Capacitor & 2.3305 & 2.0934 & 1.9517 & 1.8699 \\
Li-ECRAM & 2.3139 & 2.0665 & 1.9424 & 1.8637 \\
MO-ECRAM & 2.2937 & 2.0675 & 1.9516 & 1.8792 \\
PCM & 2.9107 & 2.8736 & 2.8732 & 2.7889 \\
OM ReRAM & 2.4031 & 2.2654 & 2.1149 & 2.0429 \\
HfO$_2$ ReRAM & 2.4741 & 2.4451 & 2.3554 & 2.2672 \\
\bottomrule
\end{tabular}
\caption{Final Shakespeare validation loss after 5,000 iterations for method S and the matched Digital baseline, across all general RPU device models. Lower is better; each cell is one seed.}
\label{tab:appendix-device-comparison-s}
\end{table}

\begin{table}[H]
\centering
\small

\begin{tabular}{lrrrr}
\toprule
Device & L2 & L4 & L6 & L8 \\
\midrule
Digital & 2.2276 & 2.0239 & 1.9006 & 1.8422 \\
\midrule
SoftBounds & 2.4357 & 2.3168 & 2.2038 & 2.1422 \\
ReRamES & 2.5540 & 2.4946 & 2.4970 & 2.4409 \\
ReRamSB & 2.5309 & 2.5390 & 2.4911 & 2.4616 \\
Capacitor & 2.5324 & 2.5215 & 2.4885 & 2.5088 \\
Li-ECRAM & 2.4887 & 2.4689 & 2.3347 & 2.2591 \\
MO-ECRAM & 2.4411 & 2.3919 & 2.2671 & 2.1957 \\
PCM & 3.0324 & 3.6237 & 3.5720 & 3.6668 \\
OM ReRAM & 2.5392 & 2.5096 & 2.4766 & 2.5024 \\
HfO$_2$ ReRAM & 2.7441 & 2.8103 & 2.6667 & 2.6385 \\
\bottomrule
\end{tabular}
\caption{Final Shakespeare validation loss after 5,000 iterations for method SFB and the matched Digital baseline, across all general RPU device models. Lower is better; each cell is one seed.}
\label{tab:appendix-device-comparison-sfb}
\end{table}

Across all nine RPU models and four depths, S attains lower final loss than SFB in this matched one-seed comparison (Tables~\ref{tab:appendix-device-comparison-s} and~\ref{tab:appendix-device-comparison-sfb}).
The shared Digital row is a reference without analog device non-idealities; the remaining rows assess algorithm behavior across the full set of distinct, non-ideal RPU update models.
Within the fitted simulator models, MO-ECRAM has a roughly seven-fold larger nominal state count than Li-ECRAM (7,089 versus 1,000), owing to its smaller $\mathrm{dw}_{\min}$, but it also specifies a larger cycle update variation ($\mathrm{dw}_{\min,\mathrm{std}}=2$ versus $0.3$; Tables~\ref{tab:aihwkit-real-device-rpu-configs-a} and~\ref{tab:aihwkit-real-device-rpu-configs-b}).
For S, the two ECRAM models give closely matched final losses, with Li-ECRAM slightly lower at the deeper models; for SFB, MO-ECRAM is lower at every depth (Tables~\ref{tab:appendix-device-comparison-s} and~\ref{tab:appendix-device-comparison-sfb}).
This is a comparison of the specified AIHWKit fits under one training protocol, not a general ranking of fabricated ECRAM technologies.

\subsection{Robustness to Device Non-Idealities}\label{sec:robustness}
We separately perturbed four non-quantization hardware properties of the canonical S configuration to distinguish update-device variation from MVM-periphery limitations.
The experiments used LogicalMpAdamW with the SoftBoundsReferenceDevice ($\tau=1$, 1,200 states), fixed construction and data seeds of 1337, one attention head, width 48, dropout 0.2, batch size 64, context length 256, 5,000 updates, and 200-batch evaluations on character-level Shakespeare.
For each property, all settings other than the perturbation were held fixed and we repeated the protocol at $L\in\{2,4,6,8\}$.
\begin{figure*}[t]
\centering
\includegraphics[width=0.48\textwidth]{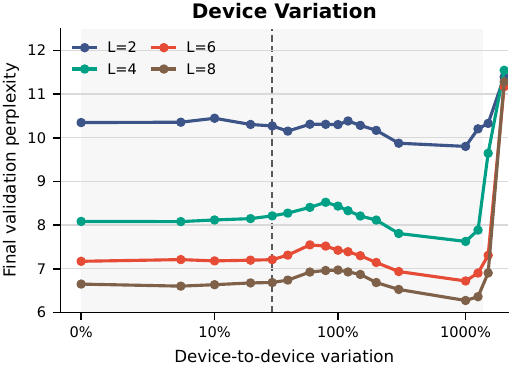}
\includegraphics[width=0.48\textwidth]{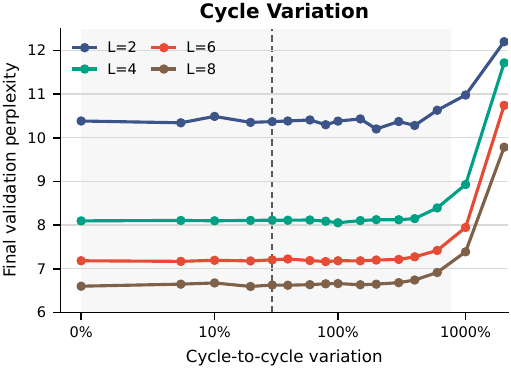}
\\[4pt]
\includegraphics[width=0.48\textwidth]{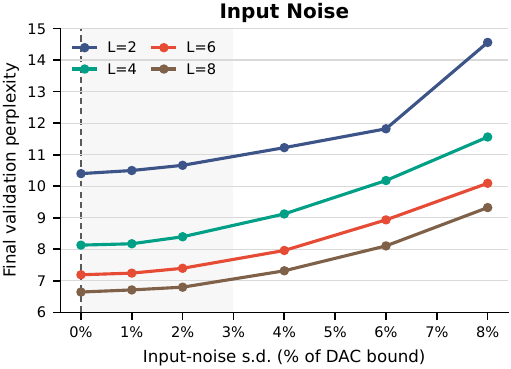}
\includegraphics[width=0.48\textwidth]{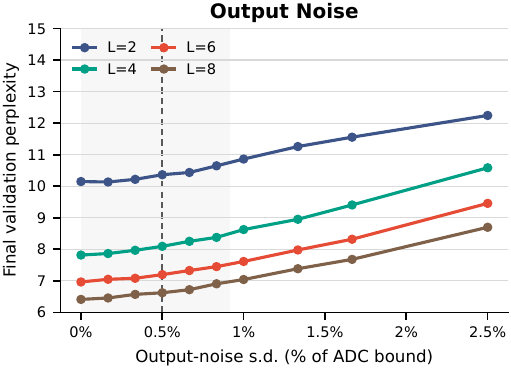}
\caption{Final Shakespeare validation perplexity after 5,000 updates for S under device variation, cycle variation, input noise and output noise, with one curve per transformer depth and all other settings held fixed.
Device variation jointly changes the device variation of the conductance limits and minimum update step; cycle variation changes the cycle update-step standard deviation.
Input- and output-noise values are matched between the forward and backward directions and are labelled as standard deviations relative to the simulator DAC and ADC bounds, respectively.
The input-noise panel ends at 8\% (the first evaluated level above 7.5\%) so that the two noise panels share a common ordinate.
Dashed vertical lines identify the corresponding AIHWKit default: 30\% for device and cycle variation, zero input noise, and 0.5\% of the ADC bound for output noise.
The gray background marks the common prefix in which all four depths remain within 10\% of their respective zero-perturbation final perplexities.
Each condition uses one fixed data and device-construction seed, so the curves show observed sensitivity in this simulator configuration rather than uncertainty estimates.}
\label{fig:appendix-s-hardware-imperfections}
\end{figure*}

Figure~\ref{fig:appendix-s-hardware-imperfections} shows the resulting sensitivity curves.
We jointly varied $w_{\max}$, $w_{\min}$, and $\mathrm{dw}_{\min}$ to model device variation, and independently varied the cycle update-step standard deviation; both sweeps extend to 2,000\% relative standard deviation.
The gray background marks the common range in which every depth remains within 10\% of its own zero-perturbation validation perplexity.
Across the four depths, performance remains near the zero-variation baseline over a substantial device-variation range but deteriorates sharply at the largest variation, showing that digital W-grad does not eliminate the finite margin for device and cycle variation.

For peripheral noise, the input-noise standard deviation spans zero to 20\% of the simulator DAC bound of 1, and the output-noise standard deviation spans zero to 2.5\% of the ADC bound of 12.
The plotted input-noise range ends at 8\%, the first evaluated point above 7.5\%, to share an ordinate with the output-noise panel; these normalized Gaussian-noise coordinates are simulator settings rather than calibrated voltage- or current-noise measurements.
Input noise degrades perplexity more rapidly over the plotted range, whereas output noise causes a smaller progressive degradation.
These single-seed measurements isolate the selected SoftBounds variations and do not establish robustness for measured devices, correlated variations, alternative device models, or other training workloads; moreover, the depths differ in parameter count, so their curve separations are not independent depth-robustness estimates.

\subsection{ADC and DAC quantization sensitivity}
\label{app:simulation-details-quantization}

We independently swept forward and backward converter resolution over 4--9
bits in a 2-layer, 1-head, width-48 Shakespeare GPT trained with MpAdamW for
5,000 iterations (seed 1337). The 36 ADC conditions set only the forward and
backward output resolutions and retained the default approximately 7-bit DAC;
the 36 DAC conditions set only input resolutions and retained the default
approximately 9-bit ADC. The two groups therefore have different fixed
counterparts and should be compared by their within-group trends.

Figure~\ref{fig:appendix-quantization} shows that final validation loss across the ADC grid ranged from 2.326 (8-bit forward,
8-bit backward) to 3.527 (4-bit forward, 4-bit backward), with the most severe
degradation at 4-bit forward ADC. Across the DAC grid, values ranged from
2.273 (8-bit forward, 9-bit backward) to 2.529 (5-bit forward, 4-bit
backward), including all 4-bit DAC conditions. These one-seed results show that the S training algorithm is more sensitive to output (ADC) quantization than to input (DAC) quantization in this configuration, especially in the forward pass.
\begin{figure*}[t]
\centering
\includegraphics[width=0.8\textwidth]{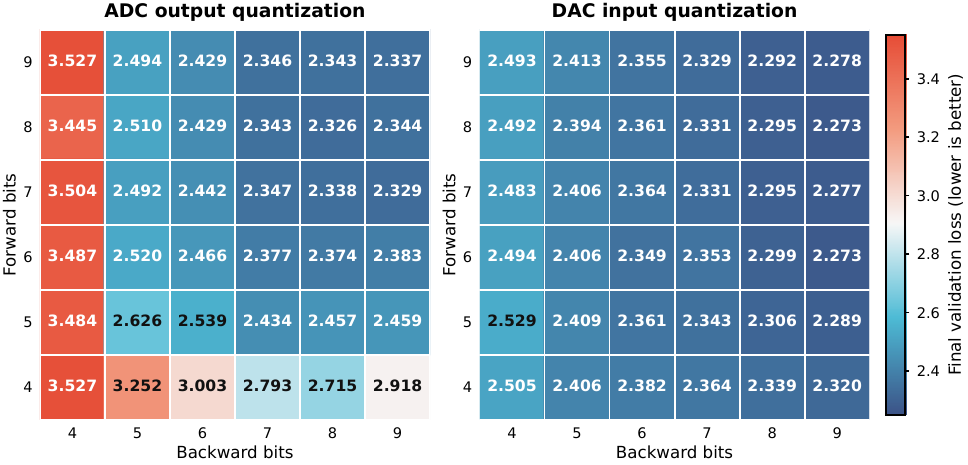}
\caption{Final Shakespeare validation loss after 5,000 iterations for method S under converter-resolution sweeps.
Rows encode forward bits and columns backward bits; lower loss is better.
The ADC panel sweeps the forward and backward output resolution while holding DAC at its default resolution, and the DAC panel analogously sweeps input resolution while holding ADC at its default resolution.
}
\label{fig:appendix-quantization}
\end{figure*}

We also swept matched DAC/ADC resolution for the SFB periphery at depths
2, 4, 6, and 8 under the same protocol. Table~\ref{tab:appendix-sfb-matched-bits}
reports final validation loss for the resulting 28 conditions.

\begin{table}[H]
\centering

\begin{tabular}{lrrrr}
\toprule
Bits & L2 & L4 & L6 & L8 \\
\midrule
3 & 2.9552 & 2.9166 & 2.9166 & 2.8549 \\
4 & 2.5548 & 2.5036 & 2.4457 & 2.3802 \\
5 & 2.4658 & 2.4212 & 2.2929 & 2.2218 \\
6 & 2.4379 & 2.3800 & 2.2233 & 2.1778 \\
7 & 2.4362 & 2.3370 & 2.2036 & 2.1418 \\
8 & 2.4212 & 2.3420 & 2.1939 & 2.1389 \\
9 & 2.4139 & 2.3103 & 2.1654 & 2.1250 \\
\bottomrule
\end{tabular}
\caption{Final Shakespeare validation loss after 5,000 iterations for SFB with matched DAC and ADC resolution.}
\label{tab:appendix-sfb-matched-bits}
\end{table}

Expected \#MVM is estimated from a deterministic 32-validation-batch replay and counts both base array invocations and bound-management retries per 256-token inference; it is a read-count proxy rather than a simulator-internal latency or energy measurement.
The largest-width read-count comparisons use the extended training endpoints and do not establish a general monotonic dependence of savings on width.

\subsection{Tied token embedding and language-head mapping}
\label{app:simulation-details-tied-head}

We compared the legacy digital tied token embedding/language head with a
fully analog mapping on character-level Shakespeare. Both conditions used a
two-layer, one-head GPT with embedding dimension 48, batch size 64, context
length 256, dropout 0.2, and the default 5,000-update schedule with a fixed
seed (1337). We used the mixed-precision \texttt{MpAdamW} optimizer and a
\texttt{Softbounds} device ($\tau=1$, 1,200 states) with imperfect I/O. Each
evaluation used 200 batches and occurred every 250 updates.

With the analog mapping, the tied matrix is represented by one analog tile:
the language head uses its forward direction and the token embedding uses its
reverse direction. The digital-head control retains the legacy digital tied
embedding/output matrix; all transformer linear layers remain analog in both
conditions.
The final-LayerNorm ablation retains normalization without learned affine parameters.

The analog tied-matrix mapping increased final validation loss by 0.0922
relative to the digital-head control. Its late-training iteration time was
12.326~ms versus 11.398~ms, an increase of 0.928~ms (8.1\%).
Disabling the final LayerNorm affine decreased validation loss by 0.0056
(2.4632 to 2.4576), while increasing late-training time by 0.301~ms (2.4\%).
These are single-seed results for a small model: they establish observed
trade-offs for this configuration, rather than multi-seed estimates of convergence.

The fully digital baseline reached a final validation loss of 2.2287 and a
late-training iteration time of 3.961~ms. It differs in all linear layers and
the optimizer path, so it should be interpreted as an all-digital reference,
not as a tied-matrix-only ablation.

\section{AIMC Component Ablation}
\label{app:profile-component-ablation}

Here, a \emph{profile} is a complete, named training configuration that fixes the initialization, weight representation, optimizer coordinate system, array-transfer rule, and forward/backward I/O path.
The ablation compares these profiles and controlled variants of them, rather than treating each row as a change to a single scalar hyperparameter.
Table~\ref{tab:profile-ablation-screen} reports the resulting validation losses.
All runs use the Shakespeare setup, digital positional embeddings, seed 1337, and 5,000 updates; each evaluation averages 200 held-out mini-batches.
This one-seed screen identifies qualitative trends but does not estimate cross-seed uncertainty.

\begin{table}[t]
\centering
\footnotesize
\renewcommand{\arraystretch}{0.96}
\setlength{\tabcolsep}{4pt}

\begin{tabular}{@{}lrrrr@{}}
\toprule
Profile & $L=2$ & $L=4$ & $L=6$ & $L=8$ \\
\midrule
\multicolumn{5}{@{}l}{\textit{Pure-digital initialization controls}} \\
Digital & 2.2276 & 2.0238 & 1.9006 & 1.8422 \\
Digital-I & 2.2355 & 2.0073 & 1.9018 & 1.8258 \\
\midrule
\multicolumn{5}{@{}l}{\textit{Original stack and factorization of S (analog)}} \\
MpSGD & 3.4444 & 3.4605 & 3.4571 & 3.4403 \\
Native-map Adam ($\omega_{\mathrm{map}}=1$) & 2.4576 & 2.4279 & 2.3902 & 2.3472 \\
Native-map Adam ($\omega_{\mathrm{map}}=0.4$) & 2.4446 & 2.3145 & 2.1863 & 2.1255 \\
O & 2.4627 & 2.4190 & 2.3236 & 2.2860 \\
O-WS & 2.4039 & 2.3245 & 2.1831 & 2.1473 \\
I & 2.4477 & 2.4111 & 2.3142 & 2.2206 \\
R & 2.4416 & 2.3922 & 2.2654 & 2.2221 \\
S & 2.3460 & 2.0902 & 1.9676 & 1.8914 \\
S-NW & 2.3358 & 2.0987 & 1.9748 & 1.8974 \\
\midrule
\multicolumn{5}{@{}l}{\textit{Complete algorithm and I/O upper control (analog)}} \\
\textbf{S-OUT} & 2.3105 & 2.0969 & 1.9657 & 1.8897 \\
S-PIO & 2.2205 & 2.0209 & 1.9202 & 1.8484 \\
\midrule
\multicolumn{5}{@{}l}{\textit{Periphery replacements applied to O (analog)}} \\
F & 3.6390 & 3.5855 & 3.9767 & 3.6345 \\
FB & 3.2442 & 3.3321 & 3.3054 & 3.2136 \\
\midrule
\multicolumn{5}{@{}l}{\textit{Cumulative periphery replacements applied after S (analog)}} \\
SF & 2.4058 & 2.2066 & 2.1036 & 2.0502 \\
SFB & 2.4105 & 2.2754 & 2.1385 & 2.0724 \\
\bottomrule
\end{tabular}
\caption{Integrated profile ablation: validation loss after 5,000 updates (lower is better).
All entries use seed 1337 and evaluations averaged over 200 held-out mini-batches.
All analog profiles retain a digital positional embedding; the definitions and interventions are given in the surrounding text.}
\label{tab:profile-ablation-screen}
\end{table}

\paragraph{Digital.}
Digital is the pure-PyTorch AdamW reference with nanoGPT's default initialization and no analog conversion.

\paragraph{Digital-I.}
Digital-I changes only the Digital initialization to the width-stable standard deviation $0.02\sqrt{768/D}$, including the residual-projection correction, and therefore tests whether that initialization is intrinsically advantageous without AIMC effects.

\paragraph{MpSGD.}
MpSGD is the mixed-precision computational-memory training method of \cite{nandakumar2020mixed}.
It executes forward and backward MVMs on the analog array, accumulates the high-precision SGD weight update in a digital buffer, and programs the array only when the accumulated update crosses the device-update threshold.
For this matched screen, MpSGD uses O's initialization and forward/backward I/O configuration together with vanilla SGD without momentum or adaptive preconditioning.
Because O instead uses MpAdamW, their difference is not an optimizer-matched test of digital weight-gradient accumulation alone~\cite{loshchilov2019decoupled,gupta2018shampoo}.

\paragraph{Native-map Adam ($\omega_{\mathrm{map}}=1$).}
This profile combines S's logical initialization and AdamW schedule with AIHWKit's native virtual mapping $W=a\widehat W$, choosing $a$ so that the initial logical AbsMax fills the physical bound.
Forward reads compute $aA_f(\widehat W x)$, backward reads apply $A_b(ae)$, and the moments, residual, and pulse quantum remain in the native mapped-tile coordinate system.

\paragraph{Native-map Adam ($\omega_{\mathrm{map}}=0.4$).}
This matched-headroom variant keeps the same native mapped forward, backward, and update transforms but maps the initial logical AbsMax to $0.4$ of the physical bound.
It therefore tests array occupancy within native mapping rather than implementing S's logical optimizer semantics.

\paragraph{O.}
O is the unmodified analog stack: nanoGPT initialization, native-conductance MpAdamW, learned columnwise output scaling, and the original AIHWKit forward/backward noise and bound management.
Turning off every proposed replacement recovers O within the scope of this ablation.

\paragraph{O-WS.}
O-WS keeps O's initialization, conductance-domain optimizer, learned output gain, and periphery, but enables AIHWKit's scalar AbsMax weight mapping with $\omega_{\mathrm{map}}=0.3$.
It isolates whether added physical headroom within the original stack is sufficient to explain S.

\paragraph{I.}
I changes only O's initialization to the width-stable rule $0.02\sqrt{768/D}$, retaining native MpAdamW, the learned columnwise output scale, and the original I/O path.

\paragraph{R.}
R replaces O's native parameterization and optimizer with the fixed-$s$ logical representation and LogicalMpAdamW, while restoring nanoGPT's default $0.02$ initialization and its $0.02/\sqrt{2L}$ residual-projection scaling.
It isolates the logical-update system without the width-stable initialization.

\paragraph{S.}
S combines I and R: it initializes conductances consistently with $W=s(\widehat W-\widehat W^\diamond)$ and accumulates the unnormalized outer-product gradient in logical-weight units.
LogicalMpAdamW keeps its moments, $\epsilon$, schedule, and pending residual in the same logical coordinate system, converts each desired increment using $\Delta\widehat W=\Delta W/s$, and sends it through the open-loop pulse-transfer operation without explicit weight readback.
The logical parameterization and optimizer are treated as one replacement because changing the optimizer alone would not preserve these logical-space semantics.

\paragraph{S-NW.}
S-NW applies the read-scale warm-up inspired by \cite{nandakumar2020mixed}: the effective logical read scale progresses from $0.7s$ to $0.85s$ and then $s$ over the first two nominal epochs, while the base write scale and LogicalMpAdamW coordinate remain fixed.

\paragraph{S-OUT.}
S-OUT is the complete algorithm evaluated in this work.
It combines S's logical parameterization and LogicalMpAdamW update path with native input noise management, forward bound management, and dimension-aware, matrix-specific ADC output rails.
Ordinary forward and backward reads use $c_{\mathrm{out}}(\tau/\omega)\sqrt{D_{\mathrm{in}}}$ and $c_{2,b}(\tau/\omega)\sqrt{D_{\mathrm{out}}}$, respectively, while the L2-normalized tied-head backward rail remains $O(\tau/\omega)$.

\paragraph{S-PIO.}
S-PIO keeps S's initialization, representation, and update path but makes both forward and backward I/O perfect.
It is an upper control for the loss attributable to residual DAC, ADC, and read-path nonidealities.

\paragraph{F.}
F keeps O's conductance-domain initialization, optimizer, and learned output scale but replaces ordinary forward management by explicit RMS normalization and static rails.
For a raw activation $x$, it computes $r_x=\sqrt{D^{-1}\lVert x\rVert_2^2}$, quantizes $x/r_x$ against the fixed input rail $c_{\mathrm{in}}$, and rescales the analog output by $r_x$ through O's learned output-scale path.
Native noise and bound management are disabled, and the static output rails are calibrated from O's physical initialization scale.
Unlike per-vector AbsMax management, this rule clips when $\lVert x\rVert_\infty>c_{\mathrm{in}}r_x$, so converter precision alone cannot remove overload error.

\paragraph{FB.}
FB adds an explicit backward path to F.
For output error $e\in\mathbb R^M$, it forms $\rho_e=\sqrt{M^{-1}\lVert e\rVert_2^2}$, quantizes $e/\rho_e$ against a static input rail, and rescales the analog input-gradient output by $\rho_e$ in O's native conductance coordinate system.

\paragraph{SF.}\label{sec:distribution-matched-mvm}
SF applies F's explicit forward RMS normalization after S.
For $z=x/r_x$ with $r_x=\sqrt{D^{-1}\lVert x\rVert_2^2}$, it computes
\begin{equation}
y=s r_x\,A_f(z),
\label{eq:forward-rms-normalization}
\end{equation}
where $A_f$ is the imperfect DAC--tile--ADC operation, the DAC rail is $c_{\mathrm{in}}$, and the ADC rail is $c_2(\tau/\omega)\sqrt D$.
The physical read changes activation transport only: the logical outer-product gradient is still accumulated from the original signed pair $(x,e)$, and the LogicalMpAdamW update path is unchanged.

\paragraph{SFB.}
SFB adds FB's explicit backward normalization and static rails to SF.
For $\bar e=e/\rho_e$ with $\rho_e=\sqrt{M^{-1}\lVert e\rVert_2^2}$, it recovers the input gradient as $g_x=s\rho_eA_b(\bar e)$ using rails $c_b$ and $c_{2,b}(\tau/\omega)\sqrt M$.
It therefore replaces both ordinary read directions after S while leaving logical gradient accumulation and open-loop array transfer unchanged.

\paragraph{Main observations.}
The O/I/R/S factorization shows that initialization or logical parameterization alone yields only a limited improvement, whereas their combination in S improves consistently with depth.
Digital and Digital-I remain close, indicating that the interaction is specific to analog conversion rather than a general advantage of the width-stable initialization.
Native weight mapping and added headroom help the original stack, but they do not reproduce S, which points to the importance of keeping gradients, optimizer state, and transfer residuals in a common logical coordinate system.

\paragraph{Read-path and periphery observations.}
S-OUT, the complete algorithm, improves on S, and S-PIO improves it further, showing that read-path nonidealities remain after the logical-update problem is stabilized.
By contrast, the explicit F/B periphery does not improve its corresponding O- or S-based profiles at the default precision, so its effect should be interpreted as conditional on the conductance scale, parameterization, and optimizer rather than as a universal main effect.
The bound-management screen also leaves S essentially unchanged when iterative retries are disabled but native input noise management is retained, and a deterministic replay observes no forward-ADC retry.
This replay omits unexposed stochastic tile noise, so it establishes operating-point headroom rather than zero clipping probability; bound management should not be removed at larger scale without a comparable occupancy check.

\section{Optimizer Dependence of the Logical Update}
\label{app:optimizer-s-sfb}

The S replacement couples a logical parameterization to a digital optimizer and quantized open-loop transfers to the analog array, so it is important to test whether its improvement is specific to AdamW.
We therefore repeat the profile-ablation protocol for AdamW, the matrix-orthogonalizing Muon, Moonlight-scaled Muon, sign-momentum Lion, matrix-preconditioned Shampoo, factored-second-moment Adafactor, and RMSprop~\cite{loshchilov2019decoupled,gupta2018shampoo,liu2025moonlight,tieleman2012rmsprop,gupta2018shampoo,shazeer2018adafactor,jordan2024muon}.
For each optimizer, we additionally run a matched pure-digital control: it uses the S/SFB width-stable initialization and every other training setting unchanged, but performs no AIMC conversion.
All methods retain the same width-stable initialization, logical scale, SoftBounds device, converter settings, digital positional embedding, seed, cosine schedule, and 5,000-update Shakespeare protocol as the initialization-map comparison in Figure~\ref{fig:initial-native-versus-s-scaling}.
The comparison changes only the digital update rule that accumulates in logical-weight units before the identical array-transfer operation.
Muon uses its standard matrix update with an Adam-style auxiliary update for non-matrix parameters, while Shampoo uses two-sided matrix preconditioners and Adafactor uses factored second-moment state.
For Moonlight, we retain its matrix-shape-calibrated update scale and its AdamW-style auxiliary update, but set decoupled weight decay to zero in all three profiles: applying nonzero decay directly to an open-loop AIMC weight would require a readback-and-rewrite operation outside this matched protocol.

\paragraph{Update formulas.}
Let $G_k$ denote the stochastic gradient at iteration $k$ in logical-weight units, and let $\Delta W_k$ be the update returned by an optimizer.
With the zero-decay setting used in Table~\ref{tab:optimizer-s-sfb}, every logical weight or residual buffer is updated as $W_{k+1}=W_k+\Delta W_k$.
For all elementwise operations below, division, square root, and powers are elementwise; $\operatorname{RMS}(X)=\sqrt{\operatorname{mean}(X\odot X)}$.
For a matrix gradient, let $D_{\mathrm{out}}$ and $D_{\mathrm{in}}$ be its row and column dimensions after flattening all trailing dimensions.
For the matrix-sign (polar) map used by Muon and Moonlight, we define
\begin{equation}
\operatorname{msign}(X)=UV^{\mathsf T},\qquad X=U\Sigma V^{\mathsf T},
\label{eq:matrix-sign}
\end{equation}
where $X=U\Sigma V^{\mathsf T}$ is an SVD of $X$.
Thus, unlike the elementwise $\operatorname{sgn}$ used by Lion, $\operatorname{msign}$ replaces the singular values of a matrix direction by one while retaining its left and right singular vectors.
In practice, we approximate $\operatorname{msign}$ with five bfloat16 Newton--Schulz iterations rather than forming an SVD~\cite{liu2025moonlight}.
The implemented rules are
\begin{align*}
\text{AdamW:}\quad &m_k=\beta_1m_{k-1}+(1-\beta_1)G_k,\qquad v_k=\beta_2v_{k-1}+(1-\beta_2)G_k\odot G_k,\\
&\Delta W_k=-\eta_k\frac{m_k/(1-\beta_1^k)}{\sqrt{v_k/(1-\beta_2^k)}+\epsilon}.\\[0.35em]
\text{Lion:}\quad &\Delta W_k=-\eta_k\operatorname{sgn}\!\left(\beta_1m_{k-1}+(1-\beta_1)G_k\right),\qquad m_k=\beta_2m_{k-1}+(1-\beta_2)G_k.\\[0.35em]
\text{RMSprop:}\quad &v_k=\rho v_{k-1}+(1-\rho)G_k\odot G_k,\qquad \Delta W_k=-\eta_k\frac{G_k}{\sqrt{v_k}+\epsilon}.\\[0.35em]
\text{Adafactor:}\quad &a_k=\beta_2a_{k-1}+(1-\beta_2)\operatorname{mean}_j(G_k\odot G_k),\\
&b_k=\beta_2b_{k-1}+(1-\beta_2)\operatorname{mean}_i(G_k\odot G_k),\\
&V_k=\left(a_k/\operatorname{mean}(a_k)\right)b_k^{\mathsf T},\qquad D_k=G_k/\max\!\left(\sqrt{V_k},\epsilon\right),\\
&\Delta W_k=-\eta_k\frac{D_k}{\max\!\left(1,\operatorname{RMS}(D_k)\right)}.\\[0.35em]
\text{Shampoo:}\quad &L_k=\beta L_{k-1}+(1-\beta)G_kG_k^{\mathsf T},\\
&R_k=\beta R_{k-1}+(1-\beta)G_k^{\mathsf T}G_k,\qquad \Delta W_k=-\eta_kL_{k,\epsilon}^{-1/4}G_kR_{k,\epsilon}^{-1/4}.\\[0.35em]
\text{Muon:}\quad &m_k=\beta m_{k-1}+(1-\beta)G_k,\qquad p_k=(1-\beta)G_k+\beta m_k,\\
&\Delta W_k=-\eta_k\sqrt{\max\!\left(1,D_{\mathrm{out}}/D_{\mathrm{in}}\right)}\,\operatorname{msign}(p_k),\\[0.35em]
\text{Moonlight:}\quad &\Delta W_k=-0.2\eta_k\sqrt{\max\!\left(D_{\mathrm{out}},D_{\mathrm{in}}\right)}\,\sqrt{\max\!\left(1,D_{\mathrm{out}}/D_{\mathrm{in}}\right)}\,\operatorname{msign}(p_k).
\end{align*}
The Nesterov-style momentum is explicit in $p_k$: after forming the momentum buffer $m_k$, the matrix direction contains a further current-gradient term.
Equivalently, $p_k=\beta^2m_{k-1}+(1-\beta^2)G_k$; without Nesterov acceleration, the corresponding direction would be $m_k$ instead.
For Shampoo, $X_{k,\epsilon}^{-1/4}$ denotes the inverse quarter root obtained from the eigendecomposition of $X_k$ after clipping each eigenvalue below by $\epsilon$.
Muon and Moonlight apply their matrix rules only to matrix parameters, and use the AdamW rule above for vectors, embeddings, and scalar parameters.
For non-matrix parameters, Shampoo uses the RMSprop direction and Adafactor uses its unfactored second-moment variant.
For an analog context, $W_k$ is the pending logical residual rather than the live conductance matrix, and the resulting increment $\Delta W_k$ is subsequently sent through the same quantized open-loop array-transfer operation as in S or SFB.
AdamW is the adaptive-moment reference, Muon conditions hidden-matrix update geometry, Moonlight adds shape-calibrated Muon update scaling, Lion replaces magnitude adaptation by a signed momentum direction, Shampoo models matrix correlations, Adafactor reduces second-moment storage, and RMSprop supplies a first-moment-free coordinate-adaptive control.
This is a fixed-hyperparameter one-seed screen rather than a per-optimizer tuning study, so it assesses optimizer compatibility with matched digital, S, and SFB data paths but does not establish a universal ranking.
Table~\ref{tab:optimizer-s-sfb} reports the final losses.

\input{tables/optimizer-s-sfb-results.tex}

Muon gives the lowest loss at every depth in the matched Digital control, whereas Lion is best at every depth under S.
Under SFB, Lion is best at $L=2$, but Moonlight is best at $L=4$, $6$, and $8$; its shape-calibrated Muon scaling is therefore particularly compatible with the explicit SFB periphery in this fixed setting.
Shampoo, Adafactor, and RMSprop are consistently worse in this unchanged hyperparameter setting.
Replacing S's managed reads with the explicit SFB periphery increases final loss for every matched optimizer--depth pair, but optimizer choice materially changes the size of that gap.
Thus, the benefit of the logical update path is not restricted to AdamW, while the S-versus-SFB ordering persists across the diverse digital update rules tested here.

\section{Finite-Precision and Asymptotic Bounds for Norm--Rail Alignment}
\label{app:max-norm-alignment}

\subsection{Periodic averaging for norm-normalized directions}
The lower bound must accommodate norm orders and rail multipliers that vary with precision.
We first establish the averaging property needed for such parameter sequences.
For $D\geq2$, let
\[
F_{j,\sigma}=\{u\in[-1,1]^D:u_j=\sigma,\ |u_i|<1\text{ for }i\ne j\},
\qquad j=1,\ldots,D,\quad \sigma\in\{-1,1\}.
\]
Let $\nu$ be a probability measure supported on the union of these open faces, with a density relative to $(D-1)$-dimensional Lebesgue measure on each face.
Write $J(u)=j$ on $F_{j,\sigma}$ and $r_p(u)=\lVert u\rVert_p$, so that $1\leq r_p(u)\leq D$ and $r_\infty(u)=1$.
Convergence in $p\in[1,\infty]$ is understood through the compact parameter $1/p\in[0,1]$, with $1/\infty=0$.

\begin{lemma}[Periodic averaging with varying norm order]
\label{lem:alignment-periodic-averaging}
Define $\psi(v)=\operatorname{dist}(v,\mathbb Z)^2$, a continuous $1$-periodic function with $\int_0^1\psi(v)\,dv=1/12$.
For any sequences $p_n\to p_0\in[1,\infty]$, $t_n\to\infty$, and $b_n\in\mathbb R$,
\begin{equation}
\begin{aligned}
&\int r_{p_n}(u)^2\sum_{i\ne J(u)}
\psi\!\left(t_n\frac{u_i}{r_{p_n}(u)}+b_n\right)\nu(du)
\\
&\hspace{2cm}\longrightarrow
\frac{D-1}{12}\int r_{p_0}(u)^2\,\nu(du).
\end{aligned}
\label{eq:alignment-periodic-averaging}
\end{equation}
\end{lemma}

\begin{proof}
Fix a face $F_{j,\sigma}$ and a coordinate $i\ne j$.
With the other free coordinates held fixed, put $T_p(u_i)=u_i/r_p(u)$.
For finite $p$, differentiation gives
\begin{equation}
\frac{\partial T_p}{\partial u_i}
=\frac1{r_p(u)}\left(1-\frac{|u_i|^p}{r_p(u)^p}\right),
\qquad
\frac1{2D}\leq\frac{\partial T_p}{\partial u_i}\leq1.
\label{eq:alignment-direction-derivative}
\end{equation}
Indeed, $|u_j|=1$ and $|u_i|<1$ imply $|u_i|^p/r_p(u)^p\leq1/2$.
For $p=\infty$, $T_\infty(u_i)=u_i$ and its derivative is one.
These maps are strictly increasing with uniformly bounded inverse derivatives.
As $p_n\to p_0$, the maps and their derivatives converge on compact subsets of the open face, including when $p_0=\infty$.

Let $g_p$ be the one-dimensional density obtained by pushing the finite measure $r_p(u)^2\nu(du)$ restricted to this face through $u\mapsto T_p(u_i)$.
The change of variables justified by \eqref{eq:alignment-direction-derivative} gives
\begin{equation}
\lVert g_{p_n}-g_{p_0}\rVert_{L^1(\mathbb R)}\longrightarrow0.
\label{eq:alignment-pushforward-continuity}
\end{equation}
To see this without any smoothness assumption on the face density, first approximate it in $L^1$ by continuous functions compactly supported in the open face.
For each approximant, convergence of the maps and inverse Jacobians gives $L^1$ convergence after the change of variables, also after integrating the other free coordinates.
The error introduced by the approximation is at most $D^2$ times its $L^1$ error, uniformly in $p$, because $r_p^2\leq D^2$ and pushforward does not increase the total variation of a finite signed measure.
This proves \eqref{eq:alignment-pushforward-continuity} for the original density.

For every $g\in L^1(\mathbb R)$,
\[
\int g(v)\psi(tv+b)\,dv\longrightarrow\frac1{12}\int g(v)\,dv
\qquad(t\to\infty),
\]
uniformly in the phase $b$.
For an interval indicator this follows by splitting into complete periods and two residual intervals; finite linear combinations of interval indicators and $L^1$ approximation give the general case, since $0\leq\psi\leq1/4$.
Apply this fact to $g_{p_0}$ and use \eqref{eq:alignment-pushforward-continuity} to replace $g_{p_0}$ by $g_{p_n}$.
Summing over the $D-1$ nonmaximal coordinates on each face and then over the faces gives \eqref{eq:alignment-periodic-averaging}.
\end{proof}

\subsection{Proof of Theorem~\ref{thm:max-norm-alignment}}
\begin{proof}
\textbf{All-resolution upper bound and the one-dimensional case.}
The uniform quantizer definition \eqref{eq:uniform-converter-quantizer}, with $K$ equal intervals over $[-s,s]$, gives
\begin{equation}
|Q_{K,s}(z)-z|\leq\frac{s}{K}
\quad\text{whenever }|z|\leq s.
\label{eq:uniform-quantizer-inrange-error}
\end{equation}
For any distribution $\mathcal V$ satisfying the theorem's finite-second-moment condition, put $M=\lVert x\rVert_\infty$ and choose the smallest index $J$ attaining $|x_J|=M$.
At the max rail, $x_J$ is a reconstruction endpoint and has zero error, including under the zero-input convention.
Consequently,
\begin{equation}
\begin{aligned}
R_K(\infty,1;\mathcal V)
&=\frac{1}{D}\mathbb E_{\mathcal V}\!\left[\sum_{i\ne J}
\bigl(Q_{K,M}(x_i)-x_i\bigr)^2\right]\\
&\leq\frac{(D-1)\mathbb E_{\mathcal V}[M^2]}
{DK^2}.
\end{aligned}
\label{eq:exact-max-rail-identity-bound}
\end{equation}
Taking the infimum over $(p,A)$ also gives $r_K(\mathcal V)\leq R_K(\infty,1;\mathcal V)$.
For $D=1$, this gives $r_K(\mathcal V)=R_K(\infty,1;\mathcal V)=0$ for every $K$ and every distribution satisfying the theorem's moment condition.
This proves the one-dimensional assertion; the theorem's asymptotic ratio is asserted only for $D\geq2$.
For the remainder of the proof, fix $D\geq2$ and a distribution satisfying the theorem's moment and absolute-continuity conditions.

\textbf{Reduction to an energy-weighted direction distribution.}
The joint density implies that $M>0$ and that its maximizing coordinate is unique almost surely, while the theorem's moment condition gives $\mathbb E_{\mathcal V}[M^2]<\infty$.
Thus $0<\mathbb E_{\mathcal V}[M^2]<\infty$.
Set $U=x/M$ and define the probability measure
\begin{equation}
\nu(B)=\frac{\mathbb E_{\mathcal V}[M^2\mathbf1_{\{U\in B\}}]}
{\mathbb E_{\mathcal V}[M^2]}.
\label{eq:alignment-energy-weighted-direction}
\end{equation}
This measure is supported on the open faces used in Lemma~\ref{lem:alignment-periodic-averaging} and has a density on each face.
In fact, if $f$ is the joint density of $x$ and $u_{j,\sigma}(v)$ inserts the fixed coordinate $\sigma$ at index $j$ into $v\in(-1,1)^{D-1}$, the face density is
\begin{equation}
\frac1{\mathbb E_{\mathcal V}[M^2]}
\int_0^\infty m^{D+1}f\!\left(m u_{j,\sigma}(v)\right)\,dm.
\label{eq:alignment-face-density}
\end{equation}
Here the radial change of variables contributes $m^{D-1}$ and the energy weight contributes $m^2$.
Tonelli's theorem and $\mathbb E_{\mathcal V}[M^2]<\infty$ ensure that these face densities are integrable.
No further tail or density regularity is needed.

The identity $Q_{K,ms}(mz)=mQ_{K,s}(z)$ for $m>0$ gives
\begin{equation}
\begin{aligned}
R_K(p,A;\mathcal V)&=c_{\mathcal V}\,\mathcal E_K(p,A),
\qquad
c_{\mathcal V}:=\frac{1}{D}\mathbb E_{\mathcal V}[M^2],\\
\mathcal E_K(p,A)&:=\int
\lVert Q_{K,A r_p(u)}(u)-u\rVert_2^2\,\nu(du).
\end{aligned}
\label{eq:alignment-directional-distortion}
\end{equation}
In particular, $\mathcal E_K(\infty,1)\leq(D-1)/K^2$.

\textbf{Leading constant for the max rail.}
At $(p,A)=(\infty,1)$, there is no clipping and the maximal coordinate is reconstructed exactly.
For the remaining coordinates, the squared distance to the reconstruction grid equals $4K^{-2}\psi(Ku_i/2+K/2)$.
Lemma~\ref{lem:alignment-periodic-averaging}, with $p_n=\infty$, $t_n=K/2$, and $b_n=K/2$, yields
\begin{equation}
K^2\mathcal E_K(\infty,1)
=4\int\sum_{i\ne J(u)}\psi\!\left(\frac{Ku_i}{2}+\frac K2\right)\nu(du)
\longrightarrow\frac{D-1}{3}.
\label{eq:alignment-max-rail-leading-constant}
\end{equation}

\textbf{Lower bound for precision-dependent norm--rail choices.}
Choose approximate minimizers satisfying
\[
\mathcal E_K(p_K,A_K)
\leq\inf_{p,A}\mathcal E_K(p,A)+K^{-3}
\leq\frac{D-1}{K^2}+K^{-3}.
\]
Thus their distortion tends to zero and their distortion multiplied by $K^2$ is bounded.
The maximal coordinate gives the clipping lower bound
\begin{equation}
\mathcal E_K(p_K,A_K)
\geq\int(1-A_Kr_{p_K}(u))_+^2\,\nu(du).
\label{eq:alignment-clipping-lower-bound}
\end{equation}
Since $r_{p_K}\leq D$, this excludes any subsequence with $A_K\to0$.

Write $h_K=A_K/K$.
The finite reconstruction alphabet is a subset of the infinite lattice
$2h_Kr_{p_K}(u)(\mathbb Z-K/2)$.
Distance to that lattice therefore gives, with or without clipping,
\begin{equation}
\mathcal E_K(p_K,A_K)\geq4h_K^2\Gamma_K,
\qquad
\Gamma_K:=\int r_{p_K}(u)^2\sum_{i\ne J(u)}
\psi\!\left(\frac{u_i}{2h_Kr_{p_K}(u)}+\frac K2\right)\nu(du).
\label{eq:alignment-infinite-lattice-lower-bound}
\end{equation}

We first show that $h_K\to0$.
Otherwise, compactness in $1/p_K$ gives a subsequence with $p_K\to p_0$ and $h_K\to h_0\in(0,\infty]$; we may also fix the parity of $K$ on this subsequence.
If $h_0<\infty$, the right-hand side of \eqref{eq:alignment-infinite-lattice-lower-bound} converges by dominated convergence to the squared distance of the nonmaximal coordinates to a lattice of positive spacing.
This integral is strictly positive: on each face, \eqref{eq:alignment-direction-derivative} gives an absolutely continuous distribution for every $u_i/r_{p_0}(u)$, $i\ne J(u)$, so its probability of lying on that discrete lattice is zero.
This contradicts $\mathcal E_K(p_K,A_K)\to0$.
If $h_0=\infty$ and $K$ is even, then zero is eventually the nearest reconstruction level to every coordinate, and $\mathcal E_K(p_K,A_K)=\int\lVert u\rVert_2^2\,\nu(du)\geq1$.
If $h_0=\infty$ and $K$ is odd, every reconstruction level has absolute value at least $h_Kr_{p_K}(u)\geq h_K$, so the distortion diverges.
Both cases are again impossible, proving $h_K\to0$.

Along any subsequence with $p_K\to p_0$, Lemma~\ref{lem:alignment-periodic-averaging} now applies with $t_K=1/(2h_K)\to\infty$ and gives
\begin{equation}
\Gamma_K\longrightarrow\frac{D-1}{12}\int r_{p_0}(u)^2\,\nu(du)>0.
\label{eq:alignment-varying-rail-average}
\end{equation}
Since \eqref{eq:alignment-infinite-lattice-lower-bound} implies $K^2\mathcal E_K(p_K,A_K)\geq4A_K^2\Gamma_K$, the boundedness of $K^2\mathcal E_K$ excludes any subsequence with $A_K\to\infty$.
Together with the exclusion of $A_K\to0$, this places the multipliers in a compact subset of $(0,\infty)$ for all sufficiently large $K$.

Take a subsequence realizing the limit inferior of $K^2\mathcal E_K(p_K,A_K)$ and then a further subsequence with $p_K\to p_0$ and $A_K\to A_0\in(0,\infty)$.
The norms converge pointwise (indeed uniformly on the cube boundary), and \eqref{eq:alignment-clipping-lower-bound} forces
\begin{equation}
A_0r_{p_0}(u)\geq1\qquad\text{for }\nu\text{-almost every }u.
\label{eq:alignment-limit-rail-covers-max}
\end{equation}
Combining \eqref{eq:alignment-infinite-lattice-lower-bound}, \eqref{eq:alignment-varying-rail-average}, and \eqref{eq:alignment-limit-rail-covers-max} gives
\begin{equation}
\begin{aligned}
\liminf_{K\to\infty}K^2\inf_{p,A}\mathcal E_K(p,A)
&\geq\frac{D-1}{3}\int A_0^2r_{p_0}(u)^2\,\nu(du)\\
&\geq\frac{D-1}{3}.
\end{aligned}
\label{eq:alignment-optimized-leading-lower-bound}
\end{equation}
The $K^{-3}$ approximation error vanishes after multiplication by $K^2$.
The matching upper limit follows by choosing $(p,A)=(\infty,1)$ in \eqref{eq:alignment-max-rail-leading-constant}.
Multiplying by $c_{\mathcal V}$ in \eqref{eq:alignment-directional-distortion} gives the common leading constant:
\begin{equation}
\begin{aligned}
\lim_{K\to\infty}K^2r_K(\mathcal V)
&=\lim_{K\to\infty}K^2R_K(\infty,1;\mathcal V)
=\frac{D-1}{3D}\mathbb E_{\mathcal V}[\lVert x\rVert_\infty^2].
\end{aligned}
\label{eq:fixed-distribution-alignment-asymptotics}
\end{equation}
For $D\geq2$ this constant is positive, so dividing the two scaled distortions proves the theorem's asymptotic ratio and completes its proof.
The one-dimensional case established above also satisfies \eqref{eq:fixed-distribution-alignment-asymptotics}, with both limits equal to zero.

For use in the quantitative threshold analysis below, we also record the resulting lower bound.
For every $\varepsilon\in(0,1/3)$ there is an integer $K_0=K_0(D,\mathcal V,\varepsilon)$ such that, for all $K\geq K_0$,
\begin{equation}
r_K(\mathcal V)\geq
\left(\frac13-\varepsilon\right)\frac{D-1}{DK^2}
\mathbb E_{\mathcal V}[\lVert x\rVert_\infty^2].
\label{eq:fixed-distribution-alignment-lower-bound}
\end{equation}
For $D\geq2$ this follows from \eqref{eq:fixed-distribution-alignment-asymptotics}; for $D=1$ it holds with equality for every $K$.
\end{proof}

\begin{remark}[Scope of the regularity and precision assumptions]
The proof only uses absolute continuity of the energy-weighted direction measure on the open faces; a joint density is a sufficient condition that guarantees this property.
Absolute continuity of each marginal alone is insufficient: if $x=(Z,\ldots,Z)$ with a nondegenerate Gaussian scalar $Z$, the max rail reconstructs every coordinate exactly, while the proposed positive lower term for $D\geq2$ would remain nonzero.
Nor does the theorem provide a threshold uniform over all absolutely continuous distributions.
For example, let $x_i=\epsilon_i(1+\delta Z_i)$, where the signs $\epsilon_i$ are independent Rademacher variables and the $Z_i$ have independent smooth densities supported on $[-1,1]$, independent of the signs.
For $0<\delta<1/2$ this distribution has a joint density, but every nonmaximal coordinate is within $2\delta$ of a max-rail endpoint, giving $R_K(\infty,1;\mathcal V)\leq4(D-1)\delta^2/D$ for every $K$.
At any fixed precision this tends to zero as $\delta\to0$, whereas $(D-1)\mathbb E[M^2]/(DK^2)\geq(D-1)(1-\delta)^2/(DK^2)$.
Thus the high-resolution threshold may depend on the distribution, even within smooth bounded product distributions.
It may also depend on dimension: for independent standard Gaussian coordinates, $(D-1)\mathbb E[M^2]/D$ grows with $D$, whereas $r_K(\mathcal V)\leq1$ follows by letting $A\downarrow0$.
An unqualified positive lower constant multiplying the maximum-moment expression therefore cannot hold for all $D,K$ in that family.
\end{remark}

\subsection{A quantitative threshold under bounded variation}
\label{app:alignment-quantitative-threshold}
An explicit sufficient threshold follows if the direction densities involved in quantization have uniformly bounded variation.

\begin{assumption}[Uniform bounded variation of weighted direction densities]
\label{ass:alignment-direction-bv}
Let $D\geq2$, and suppose the joint distribution $\mathcal V$ is absolutely continuous with respect to Lebesgue measure on $\mathbb R^D$ and satisfies $\mathbb E_{\mathcal V}[\lVert x\rVert_2^2]<\infty$.
Set $M=\lVert x\rVert_\infty$, $U=x/M$, and let $J(u)$ be the unique index with $|u_{J(u)}|=1$, which exists almost surely.
Define
\[
\nu(B)=\frac{\mathbb E_{\mathcal V}[M^2\mathbf1_{\{U\in B\}}]}{\mathbb E_{\mathcal V}[M^2]},
\qquad
r_p(u)=\lVert u\rVert_p,
\qquad
m_p=(D-1)\int r_p(u)^2\,\nu(du).
\]
For each $p\in[1,\infty]$, let $f_p$ be the probability density determined, for every bounded Borel function $\phi$, by
\begin{equation}
\int_{\mathbb R}\phi(y)f_p(y)\,dy
=\frac1{m_p}\int r_p(u)^2\sum_{i\ne J(u)}
\phi\!\left(\frac{u_i}{r_p(u)}\right)\nu(du).
\label{eq:alignment-weighted-scalar-density}
\end{equation}
There is a constant $T\geq1$ such that
\begin{equation}
\sup_{p\in[1,\infty]}\operatorname{Var}_{\mathbb R}(f_p)\leq T<\infty.
\label{eq:alignment-uniform-density-variation}
\end{equation}
Here $f_p$ is supported on $[-1,1]$ and extended by zero outside its support, and $\operatorname{Var}_{\mathbb R}(f_p)$ denotes the total variation of that function, including boundary jumps, rather than the variance of a random variable.
\end{assumption}
The densities in \eqref{eq:alignment-weighted-scalar-density} exist under joint absolute continuity, as justified by the change of variables in \eqref{eq:alignment-direction-derivative}.
The additional bound \eqref{eq:alignment-uniform-density-variation} controls their concentration and variation uniformly over the norm order; it does not follow from absolute continuity or a coordinatewise tail bound alone.

\begin{corollary}[A quantitative precision threshold]
\label{cor:alignment-quantitative-threshold}
Under Assumption~\ref{ass:alignment-direction-bv}, put $H_{D,T}=2\sqrt{D-1}+24T^2$.
For every integer $K\geq H_{D,T}$,
\begin{equation}
r_K(\mathcal V)\geq
\frac{D-1}{3DK^2}
\mathbb E_{\mathcal V}[\lVert x\rVert_\infty^2]
\left(1-\frac{H_{D,T}}{K}\right).
\label{eq:alignment-bv-finite-precision-lower}
\end{equation}
Consequently, for every $\varepsilon\in(0,1/3)$, a valid choice of threshold in \eqref{eq:fixed-distribution-alignment-lower-bound} is
\begin{equation}
K_0(D,\mathcal V,\varepsilon)
=\left\lceil\frac{2\sqrt{D-1}+24T^2}{3\varepsilon}\right\rceil
=O\!\left(\frac{\sqrt D+T^2}{\varepsilon}\right).
\label{eq:alignment-explicit-precision-threshold}
\end{equation}
\end{corollary}
This is a conservative sufficient threshold, with the distribution and dimension allowed to affect $T$; no optimal dependence on these quantities is claimed.
A common bound on $T$ makes the guarantee uniform over the corresponding class at fixed $D$.
Applying this threshold to a particular tail family requires a separate estimate of its weighted direction densities.

The following scalar estimates make the averaging argument quantitative and control rail multipliers before taking an infimum.

\begin{lemma}[Averaging and lattice separation for bounded-variation densities]
\label{lem:alignment-bv-scalar-bounds}
Let $f$ be a probability density supported on $[-1,1]$, extended by zero to $\mathbb R$, with $\operatorname{Var}_{\mathbb R}(f)\leq T$.
With $\psi(v)=\operatorname{dist}(v,\mathbb Z)^2$, for every $t>0$ and $b\in\mathbb R$,
\begin{equation}
\left|\int_{\mathbb R}\psi(ty+b)f(y)\,dy-\frac1{12}\right|
\leq\frac{T}{12t}.
\label{eq:alignment-bv-averaging-remainder}
\end{equation}
Moreover, for every $h>0$ and $b\in\mathbb R$,
\begin{equation}
\int_{\mathbb R}\operatorname{dist}^2\!\left(y,h(\mathbb Z+b)\right)f(y)\,dy
\geq\frac1{12T^2(2/h+3)^2}.
\label{eq:alignment-bv-lattice-separation}
\end{equation}
\end{lemma}

\begin{proof}
The periodic function $\psi-1/12$ has a $1$-periodic primitive $\Psi$ given on $[-1/2,1/2]$ by $\Psi(v)=v^3/3-v/12$.
Its endpoint values agree, and $\lVert\Psi\rVert_\infty=1/(36\sqrt3)\leq1/12$.
Let $\mathrm D f$ denote the distributional derivative of the zero-extended density, a finite signed measure with $|\mathrm D f|(\mathbb R)=\operatorname{Var}_{\mathbb R}(f)$.
Integration by parts gives
\[
\int_{\mathbb R}\left(\psi(ty+b)-\frac1{12}\right)f(y)\,dy
=-\frac1t\int_{\mathbb R}\Psi(ty+b)\,\mathrm D f(dy).
\]
This identity may be justified using a compactly supported cutoff equal to one on a neighborhood of $[-1,1]$.
The bound on $\Psi$ proves \eqref{eq:alignment-bv-averaging-remainder}; counting boundary jumps in $\operatorname{Var}_{\mathbb R}(f)$ is essential here.

Since $f$ vanishes off $[-1,1]$, its bounded-variation representative satisfies $\lVert f\rVert_\infty\leq T$.
For points in $[-1,1]$, the nearest point of the lattice $h(\mathbb Z+b)$ lies among its points in that interval or its two nearest exterior neighbors.
These candidates form a finite nonempty set $\mathcal G$ with $N:=|\mathcal G|\leq2/h+3$, and distance to $\mathcal G$ agrees with distance to the lattice throughout $[-1,1]$.
If $Y$ has density $f$, the union of the radius-$s$ neighborhoods of these $N$ points has Lebesgue measure at most $2Ns$, so
\[
\Pr\!\left(\operatorname{dist}(Y,\mathcal G)\leq s\right)\leq2TNs.
\]
Writing $Z=\operatorname{dist}(Y,\mathcal G)$, Tonelli's theorem and the preceding bound give
\[
\mathbb E[Z^2]
=\int_0^\infty2s\Pr(Z>s)\,ds
\geq\int_0^{1/(2TN)}2s(1-2TNs)\,ds
=\frac1{12T^2N^2}.
\]
Using $N\leq2/h+3$ proves \eqref{eq:alignment-bv-lattice-separation}.
\end{proof}

\begin{proof}[Proof of Corollary~\ref{cor:alignment-quantitative-threshold}]
Use the direction measure $\nu$, the factor $c_{\mathcal V}$, and the directional distortion $\mathcal E_K(p,A)$ in \eqref{eq:alignment-directional-distortion}, and write $d=D-1\geq1$.
For each $p$, Assumption~\ref{ass:alignment-direction-bv} supplies the density $f_p$ and the normalizing mass $m_p=d\int r_p(u)^2\,\nu(du)\geq d$.
Fix an integer $K\geq H_{D,T}$ and an arbitrary pair $(p,A)$.
If $\mathcal E_K(p,A)>d/K^2$, the desired lower bound already holds for this pair, since $(1-H_{D,T}/K)/3\leq1$.
It therefore suffices to treat pairs satisfying
\begin{equation}
\mathcal E_K(p,A)\leq\frac d{K^2}.
\label{eq:alignment-competitive-pair}
\end{equation}
This restriction retains every pair that could improve on the all-resolution max-rail guarantee and does not require existence of a minimizer.

\textbf{A uniform bound on the rail multiplier.}
Let $\mathcal L_{K,A}$ be the reconstruction alphabet at rail $A$ and extend it to the infinite lattice
\[
\Lambda_{K,A}=\frac{2A}{K}(\mathbb Z-K/2),
\qquad \mathcal L_{K,A}\subset\Lambda_{K,A}.
\]
Quantizer homogeneity and \eqref{eq:alignment-weighted-scalar-density}, retaining only the nonmaximal coordinates, yield
\begin{equation}
\begin{aligned}
\mathcal E_K(p,A)
&\geq m_p\int\operatorname{dist}^2(y,\mathcal L_{K,A})f_p(y)\,dy\\
&\geq m_p\int\operatorname{dist}^2(y,\Lambda_{K,A})f_p(y)\,dy\\
&\geq\frac{m_p}{12T^2(K/A+3)^2}.
\end{aligned}
\label{eq:alignment-bv-rail-separation}
\end{equation}
The last step applies \eqref{eq:alignment-bv-lattice-separation} with $h=2A/K$ and $b=-K/2$.
Combining \eqref{eq:alignment-competitive-pair}, \eqref{eq:alignment-bv-rail-separation}, and $m_p\geq d$ gives
\[
K\leq\sqrt{12}\,T(K/A+3).
\]
Since $T\geq1$ and $K\geq H_{D,T}\geq24T^2\geq6\sqrt{12}\,T$, rearrangement yields
\begin{equation}
A\leq\frac{\sqrt{12}\,TK}{K-3\sqrt{12}\,T}
\leq2\sqrt{12}\,T.
\label{eq:alignment-bv-multiplier-bound}
\end{equation}

\textbf{Quantitative averaging of the rounding error.}
Distance to $\Lambda_{K,A}$ also gives
\[
\mathcal E_K(p,A)
\geq\frac{4A^2m_p}{K^2}
\int\psi\!\left(\frac{Ky}{2A}+\frac K2\right)f_p(y)\,dy.
\]
Apply \eqref{eq:alignment-bv-averaging-remainder} with $t=K/(2A)$ and $b=K/2$, followed by \eqref{eq:alignment-bv-multiplier-bound}, to obtain
\begin{equation}
\begin{aligned}
\mathcal E_K(p,A)
&\geq\frac{A^2m_p}{3K^2}\left(1-\frac{2AT}{K}\right)\\
&\geq\frac{A^2m_p}{3K^2}\left(1-\frac{4\sqrt{12}\,T^2}{K}\right).
\end{aligned}
\label{eq:alignment-bv-rounding-lower}
\end{equation}
The final factor is nonnegative because $K\geq24T^2>4\sqrt{12}\,T^2$.
The infinite-lattice comparison makes these inequalities valid even when the finite quantizer clips some coordinates.

\textbf{Clipping forces sufficient mean squared rail.}
The maximal coordinate gives
\[
\int(1-Ar_p(u))_+^2\,\nu(du)\leq\mathcal E_K(p,A)\leq\frac d{K^2}.
\]
The pointwise inequality $1\leq Ar_p+(1-Ar_p)_+$ and the triangle inequality in $L^2(\nu)$ imply
\[
\left(\int A^2r_p(u)^2\,\nu(du)\right)^{1/2}
\geq1-\frac{\sqrt d}{K}.
\]
The right-hand side is nonnegative because $K\geq H_{D,T}\geq2\sqrt d$.
Consequently,
\begin{equation}
A^2m_p\geq d\left(1-\frac{\sqrt d}{K}\right)^2.
\label{eq:alignment-bv-mean-rail-lower}
\end{equation}
Combining \eqref{eq:alignment-bv-rounding-lower} and \eqref{eq:alignment-bv-mean-rail-lower} now gives
\begin{equation}
\begin{aligned}
\mathcal E_K(p,A)
&\geq\frac d{3K^2}
\left(1-\frac{\sqrt d}{K}\right)^2
\left(1-\frac{4\sqrt{12}\,T^2}{K}\right)\\
&\geq\frac d{3K^2}
\left(1-\frac{2\sqrt d+4\sqrt{12}\,T^2}{K}\right)\\
&\geq\frac d{3K^2}\left(1-\frac{H_{D,T}}{K}\right).
\end{aligned}
\label{eq:alignment-bv-directional-lower}
\end{equation}
The second inequality uses $(1-x)^2(1-y)\geq1-2x-y$ for $x,y\in[0,1]$, and the last uses $4\sqrt{12}<24$.
Pairs excluded by \eqref{eq:alignment-competitive-pair} satisfy the same final lower bound, so it holds for every $p\in[1,\infty]$ and $A>0$.
Taking their infimum and multiplying by $c_{\mathcal V}$ proves \eqref{eq:alignment-bv-finite-precision-lower}.
Finally, $\varepsilon\in(0,1/3)$ and $K\geq\lceil H_{D,T}/(3\varepsilon)\rceil$ imply both $K\geq H_{D,T}$ and $H_{D,T}/K\leq3\varepsilon$.
Substitution into \eqref{eq:alignment-bv-finite-precision-lower} gives \eqref{eq:fixed-distribution-alignment-lower-bound} with the explicit threshold \eqref{eq:alignment-explicit-precision-threshold}.
\end{proof}

\subsection{Tail-dependent upper bounds for a fixed law: long version and proof of \cref{thm:sub-weibull-max-bound}}

The following bounds retain the tail dependence of the fixed-law benchmark in \eqref{eq:exact-max-rail-identity-bound}.

\begin{theorem}[Tail-dependent finite-precision bounds, long version of \cref{thm:sub-weibull-max-bound}]
\label{cor:tail-dependent-max-bound}
For any probability law $\mathcal V$ on $\mathbb R^D$, each of the following bounds holds under its stated coordinatewise assumption; neither coordinate independence nor identical marginal laws are required:
\begin{enumerate}
\item If $|x_i|\leq B$ almost surely for every $i$, then $\mathbb E_{\mathcal V}[\lVert x\rVert_\infty^2]\leq B^2$.
\item If, for some $\alpha,L_s>0$ and $M_s\geq1$, $\Pr(|x_i|>x)\leq M_s\exp(-(x/L_s)^\alpha)$ for every $i$ and $x>0$, then $\mathbb E_{\mathcal V}[\lVert x\rVert_\infty^2]\lesssim_\alpha L_s^2(\log(2M_sD))^{2/\alpha}$.
This sub-Weibull class includes sub-Gaussian tails at $\alpha=2$, Laplace-type sub-exponential tails at $\alpha=1$, and heavier stretched-exponential tails for smaller $\alpha$~\citep{vladimirova2020sub}.
\item If, for some $M,L>0$ and $\beta>2$, $\Pr(|x_i|>x)\leq M(L/x)^\beta$ for every $i$ and $x>0$, then $\mathbb E_{\mathcal V}[\lVert x\rVert_\infty^2]\leq [\beta/(\beta-2)]L^2(MD)^{2/\beta}$.
This bound requires no coordinate independence; for independent identically distributed coordinates with a regularly varying tail of index $\beta$, its $D^{2/\beta}$ dependence is sharp up to the tail's slowly varying factor~\citep{resnick1987extreme}.
\item If $\mathbb E_{\mathcal V}[|x_i|^r]<\infty$ for every $i$ and some $r>2$, then $\mathbb E_{\mathcal V}[\lVert x\rVert_\infty^2]\leq\bigl\{\sum_{i=1}^D\mathbb E_{\mathcal V}[|x_i|^r]\bigr\}^{2/r}$.
\item If $\mathbb E_{\mathcal V}[\lVert x\rVert_2^2]<\infty$, then $\mathbb E_{\mathcal V}[\lVert x\rVert_\infty^2]\leq\mathbb E_{\mathcal V}[\lVert x\rVert_2^2]$, giving $R_K(\infty,1;\mathcal V)\leq(D-1)\mathbb E_{\mathcal V}[\lVert x\rVert_2^2]/(DK^2)$.
If the total second moment is infinite, a finite squared-error guarantee requires additional assumptions, truncation, or a lower-order distortion criterion.
\end{enumerate}
\end{theorem}

\begin{proof}
The corresponding fixed-law distortion bounds follow by substituting each estimate into \eqref{eq:exact-max-rail-identity-bound}.
The bounded case is immediate.
For each remaining tail calculation, the union bound and the uniform coordinatewise tail hypothesis give $\Pr(\lVert x\rVert_\infty>t)\leq\sum_{i=1}^D\Pr(|x_i|>t)$, while the layer-cake identity gives $\mathbb E[\lVert x\rVert_\infty^2]=\int_0^\infty 2t\Pr(\lVert x\rVert_\infty>t)\,dt$.
For the sub-Weibull case, split the integral at $L_s(\log(2M_sD))^{1/\alpha}$ and use $\Pr(\lVert x\rVert_\infty>t)\leq\min\{1,M_sD\exp(-(t/L_s)^\alpha)\}$.
The change of variables $u=(t/L_s)^\alpha$, followed by the standard incomplete-gamma tail bound, yields $\mathbb E[\lVert x\rVert_\infty^2]\lesssim_\alpha L_s^2(\log(2M_sD))^{2/\alpha}$.
For the polynomial-tail case, putting $u=L(MD)^{1/\beta}$ gives
\begin{equation}
\mathbb E[\lVert x\rVert_\infty^2]
\leq u^2+2MDL^\beta\int_u^\infty t^{1-\beta}\,dt
=\frac{\beta}{\beta-2}L^2(MD)^{2/\beta}.
\label{eq:polynomial-tail-max-second-moment}
\end{equation}
Finally, for $r>2$, monotonicity of $L^q$ norms and $\lVert x\rVert_\infty^r\leq\sum_{i=1}^D|x_i|^r$ give $\mathbb E[\lVert x\rVert_\infty^2]\leq\{\mathbb E[\lVert x\rVert_\infty^r]\}^{2/r}\leq\bigl\{\sum_{i=1}^D\mathbb E[|x_i|^r]\bigr\}^{2/r}$.
The final assertion follows directly from the pointwise inequality $\lVert x\rVert_\infty^2\leq\lVert x\rVert_2^2$.
\end{proof}

\begin{remark}[Activation distributions in normalized Transformers]
The infinite-variance laws excluded by the squared-error criterion are boundary cases, rather than an assertion that practical Transformer activations are Gaussian or light-tailed.
For a token vector $h\in\mathbb R^D$, write $\mu=D^{-1}\sum_{i=1}^D h_i$, $v=D^{-1}\lVert h-\mu\mathbf 1\rVert_2^2$, and $\widetilde h=(h-\mu\mathbf 1)/\sqrt{v+\epsilon}$ for the normalized component of LayerNorm.
It obeys $\lVert\widetilde h\rVert_2^2=Dv/(v+\epsilon)\leq D$ pointwise, so its affine output $y=\gamma\odot\widetilde h+\beta$ satisfies $\lVert y\rVert_2^2\leq2D\lVert\gamma\rVert_\infty^2+2\lVert\beta\rVert_2^2$.
Consequently, for finite learned affine parameters, the LayerNorm outputs supplied to the attention and MLP input projections in our pre-LayerNorm Transformer have finite second moments independently of the data law.
This scale control does not eliminate activation outliers after residual, linear, or nonlinear operations: such structured outliers are documented in trained Transformers~\cite{bondarenko2023quantizable}.
The tail-dependent conditions above therefore remain useful for arbitrary converter inputs, while Cauchy, Student-$t$ with $\nu\leq2$, and Pareto-type laws with tail index $\beta\leq2$ serve principally as infinite-variance boundary examples rather than typical models of the normalized reads studied here.
\end{remark}

\end{document}

%% file: tables/nandakumar-mpsgd-values.tex
\newcommand{\mpsgdtiny}{--}
\newcommand{\mpsgdsmall}{--}
\newcommand{\mpsgdmedium}{--}

%% file: tables/optimizer-s-sfb-results.tex
\begin{table}[t]
\centering
\small
\begin{tabular}{llrrrr}
\toprule
Optimizer & Profile & $L=2$ & $L=4$ & $L=6$ & $L=8$ \\
\midrule
AdamW & Digital & 2.2355 & 2.0073 & 1.9018 & 1.8258 \\
 & S & 2.3422 & 2.0925 & 1.9696 & 1.8999 \\
 & SFB & 2.4419 & 2.3393 & 2.2098 & 2.1432 \\
\midrule
Muon & Digital & \textbf{2.0141} & \textbf{1.8948} & \textbf{1.7922} & \textbf{1.7252} \\
 & S & 2.3053 & 2.0798 & 1.9650 & 1.8956 \\
 & SFB & 2.3988 & 2.2972 & 2.1825 & 2.1384 \\
\midrule
Moonlight & Digital & 2.0944 & 1.9390 & 1.8305 & 1.7595 \\
 & S & 2.1934 & 2.0098 & 1.8791 & 1.8136 \\
 & SFB & 2.3347 & \textbf{2.2110} & \textbf{2.1110} & \textbf{2.0370} \\
\midrule
Lion & Digital & 2.0819 & 1.9619 & 1.9049 & 1.8208 \\
 & S & \textbf{2.0141} & \textbf{1.9583} & \textbf{1.8510} & \textbf{1.7832} \\
 & SFB & \textbf{2.3180} & 2.2269 & 2.1793 & 2.1294 \\
\midrule
Shampoo & Digital & 2.4505 & 2.3839 & 2.3341 & 2.2686 \\
 & S & 2.4712 & 2.4089 & 2.3331 & 2.2518 \\
 & SFB & 2.4865 & 2.4734 & 2.4331 & 2.4251 \\
\midrule
Adafactor & Digital & 2.2751 & 2.0538 & 1.9447 & 1.8630 \\
 & S & 2.4414 & 2.3445 & 2.2571 & 2.1942 \\
 & SFB & 2.4663 & 2.4505 & 2.4018 & 2.3513 \\
\midrule
RMSprop & Digital & 2.2551 & 2.0330 & 1.9212 & 1.8445 \\
 & S & 2.4350 & 2.3416 & 2.2390 & 2.1791 \\
 & SFB & 2.4515 & 2.4326 & 2.3946 & 2.3493 \\
\bottomrule
\end{tabular}
\caption{Optimizer screen under matched digital, S, and SFB profiles: final Shakespeare validation loss after 5,000 updates (lower is better). Rows are grouped by optimizer. Digital uses the same width-stable initialization as S/SFB but no AIMC conversion. All cells use seed 1337 and evaluation over 200 held-out mini-batches.}
\label{tab:optimizer-s-sfb}
\end{table}